\documentclass[journal,twocolumn,]{IEEEtran}

\usepackage{cite}

\ifCLASSINFOpdf
\usepackage[pdftex]{graphicx}
\DeclareGraphicsExtensions{.pdf,.jpeg,.png}
\else
\usepackage[dvips]{graphicx}
\DeclareGraphicsExtensions{.eps}
\fi
\usepackage{epsfig,epsf,epstopdf,graphicx}
\usepackage[cmex10]{amsmath}
\usepackage{amssymb}
\usepackage{bbm}
\usepackage{amsthm}
\usepackage{nicefrac}
\usepackage{hyperref}
\usepackage{xcolor}
\hypersetup{colorlinks=true}
\usepackage{tabularx}

\newtheorem{theo}{Theorem}
\newtheorem{lemma}{Lemma}

\theoremstyle{theorem}

\theoremstyle{definition}

\theoremstyle{plain}
\newtheorem{proposition}{Proposition}
\theoremstyle{plain}

\newcommand{\muv}{\boldsymbol{\mu}}

\newcommand{\epv}{{\mbox{\boldmath $\epsilon$}}}
\newcommand{\nuv}{{\mbox{\boldmath $\nu$}}}

\newcommand{\Gamb}{{\mbox{\boldmath $\Gamma$}}}

\newcommand{\Phib}{{\mathbf{\Phi}}}

\newcommand{\av}{{\bf{a}}}
\newcommand{\bv}{{\bf{b}}}
\newcommand{\cv}{{\bf{c}}}
\newcommand{\dv}{{\bf{d}}}

\newcommand{\fv}{{\bf{f}}}

\newcommand{\mv}{{\bf{m}}}
\newcommand{\nv}{{\bf{n}}}

\newcommand{\uv}{{\bf{u}}}
\newcommand{\vv}{{\bf{v}}}
\newcommand{\wv}{{\bf{w}}}
\newcommand{\xv}{{\bf{x}}}
\newcommand{\yv}{{\bf{y}}}
\newcommand{\zv}{{\bf{z}}}
\newcommand{\zerov}{{\bf{0}}}
\newcommand{\onev}{{\bf{1}}}

\newcommand{\Hb}{\mathbf{H}}
\newcommand{\Ib}{\mathbf{I}}

\newcommand{\Lb}{\mathbf{L}}
\newcommand{\Mb}{\mathbf{M}}
\newcommand{\Nb}{\mathbf{N}}

\newcommand{\Qb}{\mathbf{Q}}

\newcommand{\Sb}{\mathbf{S}}

\newcommand{\Ub}{\mathbf{U}}
\newcommand{\Vb}{\mathbf{V}}
\newcommand{\Wb}{\mathbf{W}}
\newcommand{\Xb}{\mathbf{X}}
\newcommand{\Yb}{\mathbf{Y}}
\newcommand{\Zb}{\mathbf{Z}}
\newcommand{\Sigmab}{\mathbf{\Sigma}}

\newcommand{\Real}{\mathbb{R}}

\newcommand{\Id}{\mathbbm{1}}
\newcommand{\Prob}{\mathbbm{P}}

\newcommand{\Ac}{{\cal{A}}}

\newcommand{\Ec}{{\cal{E}}}
\newcommand{\Gc}{{\cal{G}}}

\newcommand{\Nc}{{\cal{N}}}
\newcommand{\Sc}{{\cal{S}}}
\newcommand{\Oc}{{\cal{O}}}
\newcommand{\Lc}{{\cal L}}

\newcommand{\Vc}{{\cal V}}
\newcommand{\Tc}{{\cal T}}

\newcommand\normop[1]{\left\lVert#1\right\rVert}

\newcommand{\argminn}{\mathop{ \mathsf{argmin}}}

\newcommand{\tr}{\mathop{ \mathrm{tr}}}
\newcommand{\asto}[1]{#1^\ast}

\newcommand{\diag}{\mathrm{diag}}
\newcommand{\Diag}{\mathrm{Diag}}
\newcommand{\cte}{\mathrm{const}}
\newcommand{\SR}{\mathrm{SR}}
\newcommand{\rank}{\mathrm{rank}}



\usepackage{algorithm}
\usepackage{multirow}
\usepackage{algcompatible}
\usepackage[bottom]{footmisc}
\usepackage{algpseudocode}

\PassOptionsToPackage{dvipsnames}{xcolor}
\usepackage{tikz}
\usetikzlibrary{calc}

\usepackage{subcaption}
\usepackage{fixltx2e}
\usepackage{afterpage}
\usepackage{float}
\usepackage{stfloats}

\usepackage[space]{grffile}

\makeatletter

\begin{document}

%
\title{Time-Varying Graph Learning for Data with Heavy-Tailed Distribution\thanks{This work was supported by the Hong Kong GRF 16206123 research grant and the Hong Kong RGC Postdoctoral Fellowship Scheme of Project No. PDFS2425-6S05.} }

\author{Amirhossein~Javaheri~\IEEEmembership{Graduate~Student~Member},~Jiaxi~Ying,~Daniel~P.~Palomar,~\IEEEmembership{Fellow,~IEEE},~and~Farokh~Marvasti,~\IEEEmembership{Life Senior Member,~IEEE}  
\thanks{Amirhossein Javaheri is with  the Hong Kong University of Science and Technology and Sharif University of Technology (e-mail: sajavaheri@connect.ust.hk, javaheri\_amirhossein@ee.sharif.edu). Jiaxi Ying and Daniel P. Palomar are with the 
Hong Kong University of Science and Technology, Hong Kong (e-mails: jx.ying@connect.ust.hk, palomar@ust.hk), Farokh Marvasti is with 
Sharif University of Technology, Iran (e-mail: marvasti@sharif.edu).

A conference version of this paper is presented at the 2024 EUSIPCO conference \cite{javaheri_EUSIPCO}.
}
}

\date{}
\maketitle
\thispagestyle{plain}
\pagestyle{plain}

\begin{abstract}
Graph models provide efficient tools to capture the underlying structure of data defined over networks. Many real-world network topologies are subject to change over time. Learning to model the dynamic interactions between entities in such networks is known as time-varying graph learning. Current methodology for learning such models often lacks robustness to outliers in the data and fails to handle heavy-tailed distributions, a common feature in many real-world datasets (e.g., financial data). This paper addresses the problem of learning time-varying graph models capable of efficiently representing heavy-tailed data. Unlike traditional approaches, we incorporate graph structures with specific spectral properties to enhance data clustering in our model. Our proposed method, which can also deal with noise and missing values in the data, is based on a stochastic approach, where a non-negative vector auto-regressive (VAR) model captures the variations in the graph and a Student-t distribution models the signal originating from this underlying time-varying graph. We propose an iterative method to learn time-varying graph topologies within a semi-online framework where only a mini-batch of data is used to update the graph. Simulations with both synthetic and real datasets demonstrate the efficacy of our model in analyzing heavy-tailed data, particularly those found in financial markets.
\end{abstract}


%
\IEEEpeerreviewmaketitle

\section{Introduction}
\label{sec:Intro}

\IEEEPARstart{G}{raph} signal processing (GSP) is an interesting research area that combines graph theory and signal processing to model and analyze signals defined on graph structures, which has numerous real-world applications in social networks \cite{campbell_social_2013}, image processing \cite{zhang_improved_2021}, data mining, communications \cite{tamura_applications_1970}, finance \cite{cardoso_algorithms_2020}, and more. The process of inferring the structure of a graph from data is known as graph learning \cite{ortega_graph_2018}. There are generally two types of graphs, namely undirected and directed, with each type modeling different characteristics of data. An undirected graph models bilateral relationships or similarities in data, while a directed graph is commonly used to model unilateral causal dependencies \cite{dong_learning_2019}. 

There are various approaches in the literature for modeling signals via graphs. Many graph learning methods consider a probabilistic model for the data, in which a graph structure represents the statistics of the data. A foundational technique in this domain is the Gaussian Markov Random Field (GMRF) model \cite{rue_gaussian_2005}, which assumes that the data follows a multivariate Gaussian distribution. In a general GMRF, the precision matrix encodes the conditional independence relationships between variables, thereby defining the structure of the underlying undirected graph. Recently, there has been growing interest in Laplacian constrained GMRFs \cite{ying2020nonconvex,ying2021minimax,egilmez_graph_2017}, where the precision matrix is specifically constrained to be the Laplacian matrix of an undirected graph \cite{zhang_graph_2016}. 
This constraint is particularly desirable for modeling smooth signals on graphs, where a higher weight between two nodes signifies a stronger similarity between their signal values \cite{ying2020does}. 
Under this framework, the graph (the Laplacian matrix) can be inferred via maximum likelihood (ML) or maximum a posteriori (MAP) estimation. The Graphical LASSO (GLASSO) \cite{friedman_sparse_2008} is one of the early works in this regard, which was later improved by introducing structural \cite{lake_discovering_2010,  egilmez_graph_2017, 
zhao_optimization_2019, cardoso2022learning} and spectral constraints \cite{kumar_unified_2020} into the problem of learning the graph from data. Directed graph topologies, used for modeling directional dependencies, can also be learned using different approaches. Some topology identification methods are based on structural equation models \cite{kaplan_structural_2009}, while many other approaches use vector auto-regressive (VAR) models that can be represented with directed graphs \cite{songsiri_topology_2010, bolstad_causal_2011, mei_signal_2017}. There have also been some recent works that incorporate both directed and undirected graph topologies to model spatial and temporal correlations at the same time \cite{javaheri_learning_2024}.

The methodologies mentioned above are mostly applicable for learning static graphs. However, in many real-world applications, such as social networks or finance, the network structure is subject to change over time. Therefore, one of the challenges in graph learning is to learn a time-varying graph topology. There are several approaches in the literature for learning such dynamic graphs \cite{kalofolias_learning_2017, hallac_network_2017, yamada_time-varying_2020, cardoso_learning_2020, natali_online_2021-1, hamon_tracking_2013}.
However, the existing approaches lack robustness to data outliers and cannot efficiently model heavy-tailed distributions, such as those observed in financial data. Moreover, these models also fail to capture graph topologies that adhere to specific structural and spectral properties, such as $k$-component graphs. In this paper, we propose a novel framework for learning time-varying graphs that incorporates spectral and structural constraints in the problem of inferring graph topologies. Our approach is based on a stochastic model that can first characterize the statistical properties of heavy-tailed data and second be employed for data clustering. We subsequently investigate the applications of our framework for time-varying data analysis in financial markets.

\begin{figure}[!b]
    \centering
    \vspace{-10pt}
\begin{subfigure}[t]{0.35\textwidth}
\includegraphics[trim=0 35 0 13,clip,width=\textwidth]
{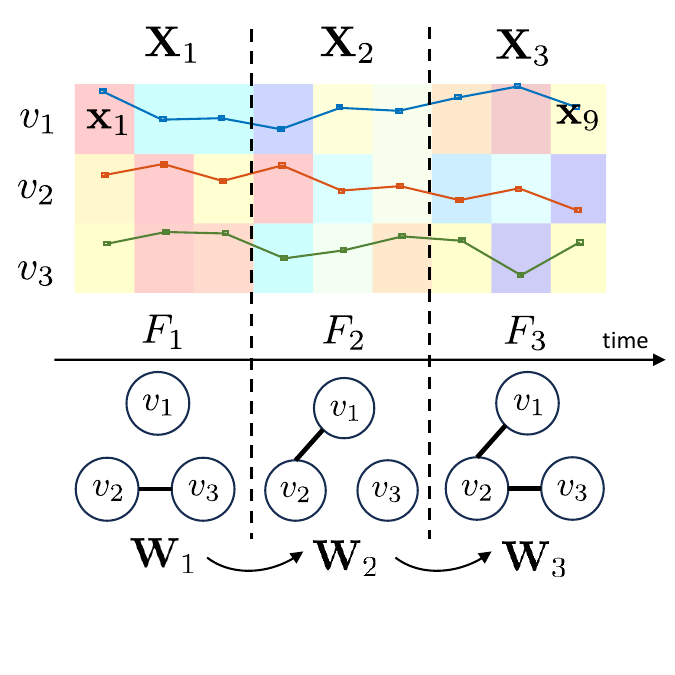}
\end{subfigure}
\caption{Illustration of  the  concept of time-varying graphs.}
\end{figure}

\begin{figure}[!t]
 \centering
\begin{subfigure}[t]{0.4\textwidth}
\includegraphics[trim=0 0 0 0,clip,width=\textwidth]
{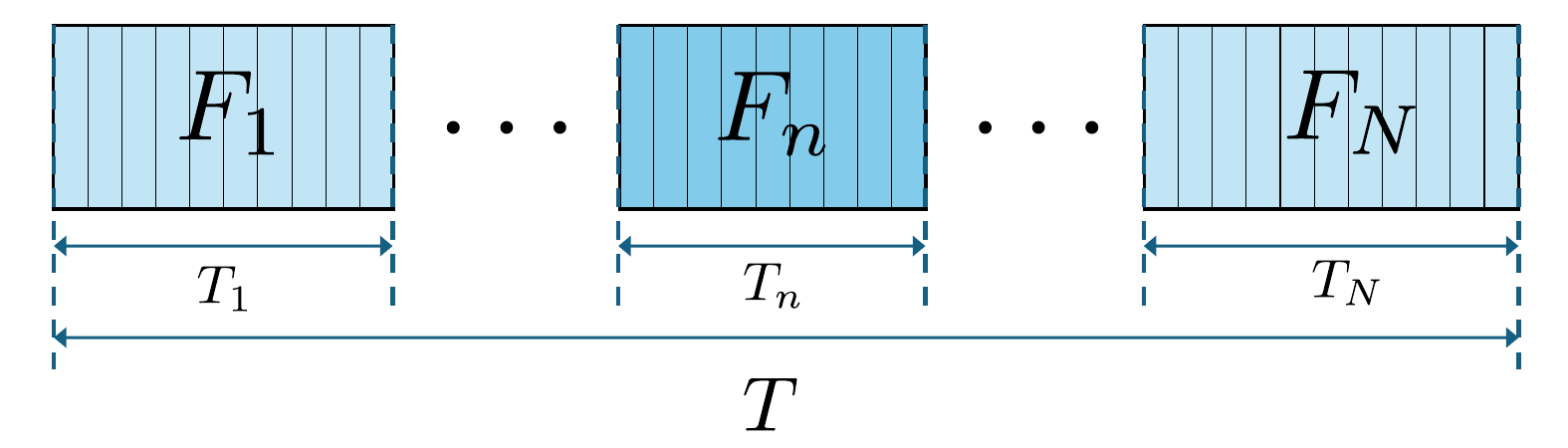}
\end{subfigure}
\caption{Illustration of  the time frames.}
\label{fig_frames}
\end{figure}

\section{Overview of Existing Works}
\subsection{Notation}
In this paper, bold lower-case letters are employed for vectors (e.g., $\mathbf{x}$) and bold upper-case letters are used to denote matrices (e.g., $\mathbf{X}$). The operator $\mathrm{det}^\star$ represents the generalized determinant (the product of non-zero eigenvalues of a matrix), and $\Id(\cdot)$ denotes the identity operator. The Hadamard (point-wise) product and division are respectively denoted by $\odot$ and $\oslash$. We also use $\normop{\mathbf{x}}_p$ for the $l_p$ vector norm (simply omitting the subscript for $p=2$) and $\normop{\mathbf{X}}_F$ for the matrix Frobenius norm. We use $\diag(\Xb)$ to denote the vector of diagonal elements of $\Xb$ and $\Diag(\xv)$ to denote a diagonal matrix with $\xv$ on its diagonals.  
The  graph Laplacian, the adjacency, and the degree operators \cite{kumar_unified_2020} are respectively denoted by $\Ac$, $\Lc$, and $\mathfrak{d}$.

   \subsection{Problem Statement}
   Assume a dynamic signal $\xv_t \in \mathbb{R}^p$ defined over a time-varying undirected graph structure, where each vertex of the graph represents an element of the signal, and the edge weights encode the (bilateral) interactions between these elements. Suppose the weights of the graph vary only at specific time instants (piece-wise constant graph). We may then segment the time indices into frames of length $T_n$, where the graph weights are assumed to remain constant within each time frame. Let $F_{n}$ denote the time indices of the $n$-th frame of data, where $n \in \{1,\ldots,N\}$. In the case of no overlap between consecutive frames, as shown in Fig. \ref{fig_frames}, we have
   \( T = \sum_{n=1}^{N} T_n \).
A time-varying undirected graph structure can subsequently be represented by $\Gc = \{\Vc, \Ec_{n}, \Wb_{n} \}$. Here, $\Vc = \{1, \ldots, p\}$ denotes the set of vertices (which remains fixed over time), $\Ec_{n} \subseteq \{ \{i, j\} \mid\, i, j \in \Vc \}$ represents the time-varying set of edges (unordered pairs of nodes connected to each other) at time frame $n$, and $\Wb_{n}$ is the (weighted) adjacency matrix. Each entry $W_{i,j}(n) \geq 0$ of $\Wb_{n}$ denotes the weight of the edge between vertex $i$ and vertex $j$ during time frame $n$. 
Alternatively, one can represent an undirected graph by the vector of all possible edge weights $\wv_{n} \in \Real_+^{p(p-1)/2}$. The edge weights can also be mapped to the adjacency matrix using the adjacency operator $\Ac$ \cite{kumar_unified_2020}, i.e., $\Wb_{n} = \Ac \wv_{n}$. Furthermore, the time-varying Laplacian matrix can be derived as $\Lb_{n} = \Diag(\Wb_{n} \onev) - \Wb_{n} = \Lc \wv_{n}$, where $\Lc$ denotes the Laplacian operator \cite{kumar_unified_2020}.
 
Considering $T$ snapshots (time measurements) of the signal vertically arranged in the columns of $\Xb = [\xv_1, \hdots, \xv_T] \in \Real^{p \times T}$. Let $\Xb_{n} = [\xv_t \mid \, t \in F_{n}]$ denote the data matrix at time frame $n$.

A time-varying graph learning problem can be generally formulated as follows:
\begin{equation}
\label{eq:TVGLASSO}
\begin{array}{cl}      
    \underset{  
 \{\Sb_{n} \mid_{n \in \Tc}\} \in \,\Omega_\Sb }{\mathsf{min}} &
 \hspace*{-8pt} \sum_{n \in \Tc} f_1(\Sb_n, \Sigmab_n) + f_2(\Sb_n) + f_3(\Sb_n, \Sb_{n-1}), 
\end{array}
\end{equation}
where $\Sb_n$ denotes the graph matrix (Laplacian or the weighted adjacency matrix), $\Omega_\Sb$ denotes the set of feasible graph matrices, $\Tc\subseteq \{1,\cdots,N\}$ denotes the set of frame indices for which the data is available, and $\Sigma_n$ denotes the data statistics matrix (e.g., sample covariance matrix). 
Here $f_1(\cdot)$ is  a fidelity criterion measuring  how well the graph matches the statistics of the data, $f_2(\cdot)$ is a regularization function used to promote  properties such as sparsity, and $f_3(\cdot)$ is a temporal consistency term formulating the smoothness of the graph variations. 

In the next part, we will discuss different choices for the objective function proposed by current methods.

\subsection{Related Works}
The notion of time-dependent graphs has roots dating back to the late 1990s \cite{harary_dynamic_1997}; however, the \emph{time-varying graph learning} methodology and concept were quite recently established. The time-varying graphical LASSO (TVGLASSO) \cite{hallac_network_2017} is one of the early works on this topic, which extends the well-known GLASSO \cite{friedman_sparse_2008} inference method to the case of time-varying topologies. Let $\Lb_n$ denote the Laplacian of a time-varying graph at frame $n$. Then, the graph learning problem in this paper is formulated as follows:
\begin{equation}
\label{eq:TVGLASSO}
\begin{array}{cl}      
    \underset{  
 \{\Lb_{n} \mid_{n = 1}^{N}\} \succ \zerov }{\mathsf{min}} &
   \hspace*{-8pt}\sum_{n=1}^{N} \Big( - T_n\log\det (\Lb_{n}) +  T_n \tr(\Sigmab_{n}\Lb_{n}) 
   \\
   &
   + \lambda \normop{\Lb_{n}}_{1,\text{off}}   \Big) + \beta\sum_{n=2}^{N} h(\Lb_{n} - \Lb_{n-1}),
\end{array}
\end{equation}
where $\Sigmab_{n}$ denotes the data statistics matrix at time frame $n$ (usually the covariance matrix) used as a similarity criterion. Moreover, $h(\cdot)$ is a regularization function utilized as a measure of the temporal smoothness of the graph variations. The TVGLASSO, however, does not incorporate the structural constraints of the CGL Laplacian matrix \cite{egilmez_graph_2017}, and the learned matrices $\{\Lb_{n} \mid_{n = 1}^N\}$ are only positive definite inverse covariance (precision) matrices. Another early work on time-varying graph topology identification is the method in \cite{kalofolias_learning_2017}.
This work incorporates the Laplacian structural constraints by reformulating the problem in terms of the weighted adjacency matrices $\{\Wb_{n} \mid_{n = 1}^N\}$ as follows:
\begin{equation}
\begin{array}{cr}        
    \underset{  
 \{\Wb_{n} \mid_{n = 1}^{N}\} \geq \zerov }{\mathsf{min}} &
  \hspace{-8pt} \sum_{n=1}^{N} \left( \tr(\Wb_{n}\Zb_{n}^\top) + f_W( \Wb_{n})  \right) 
   \\
   &
   + \gamma\sum_{n=2}^{N} h(\Wb_{n}, \Wb_{n-1}),
\end{array}
\end{equation}
where $\Zb_{n}$ denotes the distance matrix whose $(i,j)$-th entry equals the distance between signal elements $i$ and $j$ at time frame $n$. For Euclidean distance, $[\Zb_{n}]_{i,j} = \normop{\tilde{\xv}_{n_i} - \tilde{\xv}_{n_j}}$, where $\tilde{\xv}_{n_i}$ denotes the $i$-th row of the data matrix $\Xb_n$. Here, $f_W(\cdot)$ is a measure of graph smoothness, and $h(\cdot)$ is a regularization function for the temporal variations of the graph weights. Specifically, $f_W(\Wb) = -\alpha \onev^\top \log (\Wb \onev) + \beta\normop{\Wb}_F^2$ and $h(\Wb_{n}, \Wb_{n-1}) = \normop{\Wb_{n} - \Wb_{n-1}}_F^2$ are used in this paper. Another variant of the above formulation with $h(\Wb_{n}, \Wb_{n-1}) = \normop{\Wb_{n} - \Wb_{n-1}}_{1,1}$ is also considered in \cite{yamada_time-varying_2019}, which is solved using a primal-dual splitting method.

In \cite{yamada_time-varying_2020}, the static factor graph model in \cite{dong_learning_2016} is generalized to the time-varying graph factor analysis (TGFA) framework, where the graph signal at time-stamp $t$ is considered to have a Gaussian distribution as follows: 
\begin{align}
    \xv_t \sim \Nc (\zerov, \Lb_{n}^\dagger + \sigma_\epsilon^2 \Ib),\quad \forall t\in F_{n}.
\end{align}
The problem is then formulated as 
 \begin{equation}
 \label{eq:TGFA}
\begin{array}{cr}        
    \underset{  
 \{\Lb_{n}\vert_{n = 1}^{N}\} \in \,\Omega_\Lb  }{\mathsf{min}} &
   \hspace{-8pt}\sum_{n=1}^{N} \left( \tr(\Lb_{n}\Sigmab_{n}^\top) + f_L( \Lb_{n})  \right) 
   \\
   &
   + \eta\sum_{n=2}^{N} R(\Delta \Lb_{n} \odot \Hb),
\end{array}
\end{equation}
where $\Hb = \Ib - \onev\onev^\top$, and $\Omega_\Lb$ denotes the set of feasible CGL Laplacian matrices, with $ f_L(\Lb_n) = -\alpha \onev^\top \log (\diag(\Lb_n))$ $+$ $\beta \normop{\Lb_n}_{\text{off},F}^2 $.

 Choosing $R(\cdot) = \normop{\cdot}_{1,1}$ to promote the sparsity of the temporal variations, a formulation similar to the authors' prior work \cite{yamada_time-varying_2019} is yielded, which is also solved via a primal-dual splitting method. In \cite{cardoso_learning_2020}, several approaches for graph learning, including time-varying graphs, are proposed with applications in financial markets. The problem formulation for learning time-varying graphs in this paper is the same as that in the TVGLASSO method \eqref{eq:TVGLASSO}. However, here the structural constraints of the Laplacian matrices $\{\Lb_{n}\vert_{n=1}^{N}\}$ are incorporated into the formulation (as in \eqref{eq:TGFA}), and only causal batch data (past data samples) are used for graph learning.  
In \cite{natali_learning_2022}, a general framework for time-varying topology learning is introduced that can deal with online streaming data. Three types of models are studied in this paper, namely the time-varying Gaussian graphical model (TV-GGM), the time-varying smoothness-based model (TV-SBM), and the time-varying structural equation model (TV-SEM), where the first two apply to undirected graphical models, and the last one is utilized for directed topology identification. The methodology in this paper is based on a time-varying optimization framework proposed in \cite{simonetto_class_2016}. Here, the general formulation for time-varying graph learning is as follows:
 \begin{equation}
\begin{array}{cl}        
  \Lb_t^\star  = \underset{  
 \Lb  }{\mathsf{argmin}} &
   F(\Lb;t) = 
   f(\Lb;t) + \lambda\, g(\Lb;t),
\end{array}
\end{equation}
where $\Lb_t^\star$ denotes the matrix representation for the graph model (e.g., the Laplacian matrix for an undirected graph).
In this formulation, $f(\cdot)$ represents a smooth, strongly convex differentiable fidelity (similarity) measure, and $g(\cdot)$ denotes a potentially non-smooth convex regularization function. For instance, in the TV-GGM Gaussian graphical model, we have
\begin{align}
\begin{split}
    f(\Lb;t) &= -\log \det(\Lb) + \tr(\Lb\Sigmab_t),\\
    g(\Lb;t) &= g(\Lb) = {i}_{\Omega_\Lb}(\Lb) = \begin{cases}
        0, & \Lb \in \Omega_\Lb\\
        \infty, & \Lb \notin \Omega_\Lb
    \end{cases} ,
\end{split}
\end{align}
where $\Omega_\Lb = \{\Lb \succ \zerov\}$ denotes the set of positive definite matrices.
In this problem, a solution is found via recursive prediction-correction steps. In the prediction step, a quadratic second-order approximation $\hat{F}(\Lb;t+1)$ of the unobserved function $F(\Lb;t+1)$ is minimized. In the correction step, the exact function $F(\Lb;t+1)$ is optimized by exploiting the updated statistics of the data, with the new data sample $\xv_t$ being received.

In addition to time-varying graph learning approaches for undirected structures, there have been efforts to learn time-varying directed topologies arising in structural equation models, including the works in \cite{natali_online_2021-1} and \cite{baingana_tracking_2017}. 
Several other online approaches to graph topology identification have also been studied in the literature, e.g., \cite{vlaski_online_2018, shafipour_online_2020, money_online_2021, saboksayr_online_2021, sardellitti_online_2021, buciulea_online_2024}.

Most of the existing approaches for time-varying graph learning are either tailored for offline (full-batch) data, where all the samples $\{1, \ldots, T\}$ are collected for topology identification \cite{kalofolias_learning_2017, hallac_network_2017, yamada_time-varying_2020}, or they deal with online streaming data, where the graph topology is updated with every new data sample $\xv_t$ being acquired \cite{natali_online_2021-1, natali_learning_2022}. There is also a causal (batch) approach \cite{cardoso_learning_2020} that exploits all the past data frames (e.g., $1, \ldots, n$) for learning the graph at the current time frame (e.g., $n$). The offline (full-batch) and causal (batch) approaches suffer from delays and the need for large data storage, while the online approach is slow due to the high computational costs incurred by graph learning at every time stamp. This can make it impractical for real-time use unless some form of cache storage is utilized. Additionally, online approaches may struggle to efficiently capture the temporal consistency in the variations of the graph weights.

This challenge can be addressed by proposing a semi-online (mini-batch) approach in which the graph is updated using only the data samples from a single time frame. The length of the data frame \( T_n \) can then be adjusted based on how dynamic the desired graph topology is assumed to be, depending on the application.

The existing time-varying graph learning methods are also designed for Gaussian data, and they cannot efficiently deal with heavy-tailed data or data with outliers, which are very common in real-world applications, e.g., financial markets \cite{resnick_heavy-tail_2007}. Hence, these methods cannot be applied to learning time-varying topologies of markets in finance. They cannot either learn graphs with specific spectral properties that can be used for clustering, e.g., \( k \)-component graphs, which are very applicable in unsupervised machine learning (data mining) \cite{paratte_graph-based_2017}.

{\color{black} Another issue with the current approaches is that they
utilize a simple subtractive model for the graph variations, where the difference \( \wv_{n} - \wv_{n-1} \) is assumed to be smooth. However, in many real cases, this model is not sufficient and a
multiplicative factor \( \av \) as \( \wv_{n} - \av \odot \wv_{n-1} \) may be required to better model the variations.

Another drawback of the existing methods is that they rely on complete statistics of the data, while in many real-world applications (such as sensor networks), there may be missing entries in the data (due to sensor failure) or the data may be contaminated with noise (due to measurement errors). This issue has been addressed for static graph learning (e.g., \cite{dong_learning_2016, javaheri_joint_2024}) by adopting a joint signal and graph learning approach in which an additional signal denoising/imputation step is performed. However, this has not been well addressed for learning time-varying graphs. 

In the following section, we address these issues by proposing a robust predictive approach to learning time-varying graphs with specific properties that apply to heavy-tailed data.

\begin{table*}[!t]
 \scriptsize
\centering
 \caption{{Summary of various benchmark methods for time-varying graph learning in comparison to the proposed approach. The objective function comprises a temporal consistency term and a graph regularization term. In this context, \( t \) represents samples at each time instant, \( n \) denotes the index of the frame (mini-batch), and \( F_n \) refers to the time indices within the \( n \)-th frame.}
}
 \label{table:methods}
 \begin{tabular}{c|cccc} 
     & Temporal  Consistency  
     &  Graph Regularization + Data Fidelity 
     & Constraints  
     & Data \\
     \hline
     
    \multirow{2}*{TVGLASSO  \cite{hallac_network_2017} } 
    &  $\sum_{n=2}^{N} \normop{\Lb_{n} - \Lb_{n-1}}_{p,q}^q$ & 
      $ \sum_{n=1}^{N}\left(-T_n \log\det (\Lb_n)+\lambda\normop{\Lb_{n}}_{1,\text{off}}\right)$ &  
      \multirow{2}*{$\Lb_n \succ \zerov$}
      & $\Xb = [\Xb_n\vert_{n=1}^N]$\\
    &  $ p,q \in \{1,2,\infty\}$ & 
      $ +\sum_{n=1}^{N}T_n \tr(\Sigmab_{n}\Lb_{n})$ 
      &  
      & (Full-batch)\\[3pt]
     \hline   
    \multirow{2}*{Kalofolias et al.  \cite{kalofolias_learning_2017}} 
    & \multirow{2}*{ $\sum_{n=2}^{N}  \normop{\Wb_{n} - \Wb_{n-1}}_F^2$} 
    & 
     $\sum_{n=1}^{N} \left(-\alpha \onev^\top \log (\Wb_n \onev) + \beta\normop{\Wb_n}_F^2\right)$  & 
      $\Wb_n\geq \zerov, \,\Wb_n = \Wb_n^\top,$
       & $\Xb = [\Xb_n\vert_{n=1}^N]$\\
       &  
       & $+\sum_{n=1}^{N}\tr(\Wb_{n}\Zb_{n}^\top)$
       & 
      $\diag(\Wb_n)=\zerov$
       & (Full-batch)\\[3pt]
     \hline
     \multirow{2}*{Yamada et al. \cite{yamada_time-varying_2019}}   &
    \multirow{2}*{ $\sum_{n=2}^{N}  \normop{\Wb_{n} - \Wb_{n-1}}_{1,1}$} & 
      $\sum_{n=1}^{N} \left(-\alpha \onev^\top \log (\Wb_n \onev) + \beta\normop{\Wb_n}_F^2\right)$ & 
     $\Wb_n\geq \zerov, \,\Wb_n = \Wb_n^\top,$ 
   & $\Xb = [\Xb_n\vert_{n=1}^N]$ \\
      &
      & $+\sum_{n=1}^{N} \tr(\Wb_{n}\Zb_{n}^\top)$
      & $\diag(\Wb_n)=\zerov$
   & (Full-batch) \\[3pt]
     \hline \rule{0pt}{1.05\normalbaselineskip}
\multirow{1}*{TV-GGM \cite{natali_learning_2022}}  
    & $---$
    &  $ -\log \det(\Lb_t) + \tr(\Lb_t\Sigmab_t)$
    & $\Lb_t\succ \zerov$
    & \multirow{1}*{$\xv_t$ (Online)}\\[2pt]
      \hline \rule{0pt}{1.05\normalbaselineskip}
\multirow{2}*{TV-SBM \cite{natali_learning_2022}}   &
    \multirow{2}*{ $---$} & 
      $-\alpha \onev^\top \log (\Wb_t \onev) + \beta\normop{\Wb_t}_F^2$ & 
     $\Wb_t\geq \zerov, \,\Wb_t = \Wb_t^\top,$ 
   & $\xv_t$ \\
      &
      & $+ \tr\left( (\Diag(\Wb_{t}\onev) - \Wb_t)\Sigmab_{t}\right)$
      & $\diag(\Wb_t)=\zerov$
   & (Online) \\[1pt]
     \hline \rule{0pt}{1.05\normalbaselineskip}
\multirow{2}*{\multirow{2}*{Proposed}}   
    & \multirow{2}*{ \multirow{2}*{ $ \normop{\wv_{n}- \av\odot {\wv}_{n-1}}_1$} }  
     & $-  \log \det{\!^\ast} (\Lc\wv_{{n}})+\beta\normop{\wv_{n}}_0$ 
     & $\wv_n\geq \zerov,\, \mathfrak{d}\wv_n = \dv,$ 
      & $\Xb_n$ \\
      &
      & $+\frac{\nu+p}{T_n}  \sum_{t\in F_{n}}\log\left(1+\frac{\xv_{t}^\top\Lc\wv_{n}\xv_t}{\nu} \right)$
      & $\rank(\Lc\wv_n) = p-k$
   & (Mini-batch) \\[1pt]
     \hline
 \end{tabular}
 \end{table*}

\subsection{Contributions}

Our contributions can be summarized as follows:
\begin{itemize}
    \item We propose a probabilistic framework to model the signal and the weights of the graph in time-varying scenarios. Specifically, we utilize a non-negative VAR model to capture the temporal variations in the weights of the graph. Our method is based on MAP estimation of the graph model in a  semi-online approach in which the graph is only updated within frames, where the frame length can be adjusted to achieve the optimal balance between complexity and dynamics.

    \item We consider a heavy-tailed Student-\textit{t} distribution for the signal characterized by the Laplacian matrix of a time-varying graph. This distribution can efficiently model data with outliers (e.g., financial data) and it can also 
    handle Gaussian data by choosing the parameter $\nu$ large enough.
    Our method is also robust to measurement noise and missing data.


    \item We incorporate spectral and structural constraints into the problem of learning time-varying graphs. Our method can be used to learn $k$-component graphs applied to data clustering. We achieve this by imposing constraints on the node degrees and the rank of the Laplacian matrix. 

    \item We propose an iterative method with proof of convergence, to solve the problem using the alternating direction method of multipliers (ADMM), where a majorization of the original function is optimized in each subproblem.  Numerical results demonstrate that our proposed method outperforms some state-of-the-art algorithms for learning time-varying graph models, specifically from heavy-tailed data.  
\end{itemize}

\section{Proposed Approach}
In contrast to the deterministic approaches for time-varying graph learning, we propose a probabilistic framework to address the problem of learning a time-varying graph. To this end, we employ a non-negative vector autoregressive (VAR) equation to characterize the variations in the edge weights of the graph as follows:
\begin{align}
    \wv_{n} = \left( \av \odot \wv_{n-1} + \epv_{n} \right)_+, \quad n=1,\ldots, N,
\end{align}
where the positive part function \( (\cdot)_+ \) is used to ensure the weights remain non-negative. Here, 
\(\av\in \mathbb{R}_+^{p(p-1)/2}\) models the VAR coefficients vector assumed to have an exponential prior and 
\(\epv_{n}\) is a zero-mean temporally and spatially white innovation process with the Laplace distribution. The choice of the Laplace distribution is justified for promoting sparsity in the graph weights. Then, we have:
\begin{align}
\label{eq:laplace}
\begin{split}
    p(\av) &= \lambda^{p(p-1)/2} \exp\left(-\lambda \,\av^\top\onev\right), \quad \lambda> 0, \\
    p(\epv_{n}) &= \frac{1}{(2\sigma_\epsilon)^p} \exp\left(-\frac{\|\epv_{n}\|_1}{\sigma_\epsilon}\right).
\end{split}
\end{align}

We also adopt a stochastic model for the signal, presuming that $\xv_t$ for $t\in F_n$ follows a Laplacian heavy-tailed multivariate Student-\textit{t} distribution, as follows:
\begin{align}
    \label{eq:heavy-tail}
    p(\xv_t\vert\wv_{n})   \propto \mathrm{det^*}(\Lc \wv_{n})^{1/2} \left(1+\frac{\xv_t^\top \Lc \wv_{n} \xv_t}{\nu}\right)^{-(\nu + p)/2},\quad \nonumber \\
    t\in F_{n},\, \quad \nu>2. 
\end{align}

Choosing the Student-\textit{t} distribution for modeling heavy-tailed data, particularly in financial contexts, is extensively supported in the literature \cite{resnick_heavy-tail_2007}. Now, suppose we have corrupted measurements of the signal, i.e., some samples are missing, and there is also some noise. Therefore, the measurements $\yv_t$ are modeled as
\begin{align}
\begin{split} 
    \yv_t &= \mv_t \odot (\xv_t + \nv_t), \quad  t\in\{1, \hdots , T\}, \\
    \Yb &= \Mb\odot (\Xb + \Nb),
    \end{split}
\end{align}
where $\mv_t$ is a given sampling mask vector with binary elements (zeros correspond to missing samples), $\odot$ denotes the point-wise Hadamard product, and $\nv_t$ is a zero-mean i.i.d. Gaussian noise vector with distribution $\nv_t \sim \Nc(\zerov, \sigma_n^2 \Ib)$. We also have $\Yb = [\yv_t\mid_{t=1}^T]$, $\Mb = [\mv_t\mid_{t=1}^T]$, and $\Nb = [\nv_t\mid_{t=1}^T]$.

To estimate the time-varying graph weights $\wv_{n}$, we employ a maximum a posteriori (MAP) estimation through a semi-online (mini-batch) approach, where only a single data frame is utilized for graph learning.

Let \( \Yb_{n} = [\yv_t\vert  \,t\in F_{n}] = \Mb_{n}\odot (\Xb_{n} + \Nb_{n}) \) denote the matrix of the corrupted signal samples (observations) at the time frame \( n \). For inference of \( \wv_{n} \), the VAR model parameters, \( \av \), and the original (clean) signal \( \xv_t \), in a semi-online fashion, data collection is limited to the \( n \)-th time frame, i.e., we only utilize \( \Yb_{n} \). In this scenario, we may need to await the availability of \( T_n \) data samples in time frame \( n \). Nonetheless, this approach can also be adapted for online inference by setting \( T_n=1 \).
The problem can then be expressed as follows:
\begin{align}
 \label{eq:MAP_Online}
\begin{array}{cl}        
     \underset{  
 \wv_{n}\geq \zerov,\, \av\geq \zerov,\, \Xb_{n}}{\mathsf{min}} &
      -\log p(\wv_{n} ,\av, \Xb_{n}\vert {\wv}_{n-1}, \Yb_{n}, \Mb_n) \\
    \mathsf{s.t.} & \wv_{n} \in \Omega_\wv
\end{array} 
\end{align}
where 
\mbox{$
    -\log p(\wv_{n} ,\av, \Xb_{n}\vert {\wv}_{n-1}, \Yb_{n},\Mb_n)   =- \log p(\Xb_{n}\vert  \wv_{n})$ } $-\log p( \Yb_{n}\vert\Xb_{n}, \Mb_n)  
      -\log p(\wv_{n}\vert {\wv}_{n-1},\av) -\log p(\av) + \cte \nonumber
$
and \( \Omega_\wv \) represents the set of equality constraints that define the feasible region of the desired graph weights. In specific, we assume that the underlying graph structure is \( k \)-connected, and the degrees of the vertices are fixed and equal to a constant vector \( \dv \), i.e., 
\(
\Omega_\wv = \{\wv\vert\,\,    
\rank(\Lc \wv) = p-k,  \mathfrak{d}\wv  = \dv\},
\)
where \( \mathfrak{d} \) represents the degree operator, mapping the vector of edge weights to the vertex degrees \cite{de_miranda_cardoso_graphical_2021}.
 
\begin{proposition}
\label{prop:1}
Let $\wv_{n-1}$ in \eqref{eq:MAP_Online} be replaced by an estimate of the graph weights from the previous time frame, denoted as $\hat{\wv}_{n-1}$. By expanding the posterior probability for MAP estimation and simplifying, we obtain the following formulation for learning the time-varying graph:
 \begin{equation}
 \label{eq:problem_semi}
\begin{array}{cl}        
     \underset{  
 \wv_{n}\geq \zerov,\, \av\geq \zerov,\, \Xb_{n}}{\mathsf{min}} &\hspace{-5pt}
       \frac{1}{T_n\sigma_n^2}\normop{\Yb_{n}- \Mb_{n}\odot \Xb_{n}}_F^2 -  \log \det{\!^\ast} (\Lc\wv_{{n}} ) \\
       &\hspace{-5pt} +\frac{\nu+p}{T_n}  \sum_{t\in F_{n}}\log\left(1+\frac{\xv_{t}^\top\Lc\wv_{n}\xv_t}{\nu} \right) \\
       &\hspace{-5pt} +\alpha  \normop{\wv_{n}- \av\odot \hat{\wv}_{n-1}}_1 + \beta \normop{\wv_{n}}_0 + \gamma \av^\top\onev\\
    \mathsf{s.t.} &\hspace{-5pt} 
    \wv_{n} \in \Omega_\wv
\end{array}
\end{equation}
where \( \alpha = \frac{2} {T_n\sigma_\epsilon} \), \( \beta = \frac{2\log\sigma_\epsilon}{T_n} \), and \( \gamma = \frac{2\lambda}{T_n} \). 
\end{proposition}
\begin{proof}
    See appendix  \ref{App:1}.
\end{proof}

The non-convex formulation of problem \eqref{eq:problem_semi}, coupled with the interleaved equality constraints on \( \wv_{n} \), renders it challenging to solve. Nevertheless, taking advantage of splitting techniques in convex optimization, particularly ADMM \cite{boyd_distributed_2010}, a promising approach to address this problem can be devised. We begin by reformulating the problem, introducing the auxiliary variables \( \Lb_{n} = \Lc\wv_{n} \) and \( \uv_{n} = \wv_{n} - \av\odot \hat{\wv}_{n-1} \). We further incorporate an additional penalty function \( \tr\left(\Lc\wv_n\Vb_n\Vb_n^\top\right) \) with \( \Vb_n\in\Real^{p\times k},\, \Vb_n^\top\Vb_n = \Ib \) to more effectively control the rank of \( \Lc\wv_n \). Thus, the problem formulation becomes:
 \begin{equation}
 \label{eq:objective}
\begin{array}{cl}       
\begin{array}[t]{c}
    {\mathsf{min}}\\
    \scriptstyle \wv_{n}\geq \zerov,\, \av\geq \zerov,\\[-2pt]
    \scriptstyle \Xb_{n},\, \uv_{n}, \, \Lb_{n},\, \Vb_n
    \end{array} 
     & \hspace{-8pt}\begin{array}[t]{l}
    f(\wv_{n},\Xb_{n}, \av, \uv_{n}, \Lb_{n}, \Vb_n
    )\triangleq\\
    
     \hspace{10pt} - \log \det{\!^\ast} (\Lb_{n} ) + \alpha \normop{\uv_{n}}_1 + \beta \normop{\wv_{n}}_0  \\
     \hspace{10pt}+\, \frac{\nu+p}{T_n} \sum_{t\in F_{n}}\log\left(1+\frac{\xv_t^\top\Lc\wv_{n}\xv_t}{\nu} \right) \\[3pt]
     \hspace{10pt}+\, \frac{1}{T_n\sigma_n^2} \normop{\Yb_{n} - \Mb_{n}\odot \Xb_{n}}_F^2 + \gamma \av^\top \onev\\
     \hspace{10pt}+\,\eta \tr\left(\Lc\wv_n\Vb_n\Vb_n^\top\right)
      \end{array}
     \vspace{5pt} \\
    \mathsf{s.t.} &  \hspace{-8pt} \begin{array}[t]{l}\Lb_{n} = \Lc \wv_{n}, \,\, \uv_{n} = \wv_{n}-\av\odot\hat{\wv}_{n-1} ,\\
    \mathfrak{d}\wv_{n} =\dv,\,\, \rank(\Lb_{n}) = p-k, \, \Vb_n^\top\Vb_n = \Ib. 
    \end{array}
\end{array}
\end{equation}

Hence, the augmented Lagrangian for this problem yields as  follows:
\begin{align}
    &L_\rho(\wv_{n},\Xb_{n}, \av, \uv_{n}, \Lb_{n}, \Vb_n, \Phib_n,\muv_n, \zv_n)  =\nonumber \\
    & \hspace{5pt}f(\wv_{n},\Xb_{n}, \av, \uv_{n}, \Lb_{n},\Vb_n
    )\nonumber \\
   & \hspace{5pt}+ \frac{\rho}{2}\normop{\Lc\wv_{n}-\Lb_{n}}_F^2 + \langle \Lc\wv_{n}-\Lb_{n}, \Phib_{n}\rangle
    \nonumber \\
    &\hspace{5pt}+\frac{\rho}{2}\normop{\uv_{n}-\wv_{n}+\av\odot \hat{\wv}_{n-1}}^2+ \langle \uv_{n}-\wv_{n}+\av\odot \hat{\wv}_{n-1}, \muv_{n}\rangle \nonumber \\
   &\hspace{5pt} + \frac{\rho}{2}\normop{\mathfrak{d}\wv_{n}-\dv}^2 +\langle \mathfrak{d}\wv_{n}-\dv ,\zv_{n}\rangle.
\end{align}
Now, employing ADMM, we derive an iterative solution consisting of  six update steps for the primal variables $\wv_{n}$, $\Xb_{n}$, $\av$, $\uv_{n}$,  $\Lb_{n}$ and $\Vb_n$, along with three update steps for the dual variables $\Phib_{n}$, $\muv_{n}$ and $\zv_{n}$.

\subsection*{$\Lb_{n}$-update step}
The subproblem associated with the update step of $\Lb$ possesses a closed-form solution given by:
\begin{align}
\label{eq:Theta_update}
\begin{split}
   \Lb_{n}^{l+1} &= \hspace{-19pt}\underset{\Lb_{n} \succeq \zerov,\, \rank(\Lb_{n}) = p-k}{\argminn}
    \hspace{-19pt} -\log \mathrm{det^\ast}(\Lb_{n})+\frac{\rho}{2}\normop{\Lb_{n} - \Lc\wv_{n}^l - \frac{1}{\rho}\Phib_{n}^l }_F^2 \\
    &= \frac{1}{2} \Ub^l \left(\Gamb^l + \left({ {\Gamb^{l}}^ 2 + \frac{ 4}{\rho} \Ib}\right)^{1/2}\right)\Ub^{l \top}.
    \end{split}
\end{align}

Here, $\Gamb^l$ is a diagonal matrix comprising only the largest $p-k$ eigenvalues of $ \Lc \wv_{n}^l + \Phib_{n}^l/\rho $, with their corresponding eigenvectors contained in $ \Ub^l\in \Real^{p\times (p-k)} $.

\subsection*{$\wv_{n}$-update step}
The subproblem related to  the update step  of $\wv$ is expressed as follows: 
\begin{equation}
\begin{split}
    \wv_{n}^{l+1} \!=\argminn_{\wv_{n} \geq \zerov} &\hspace{3pt}\frac{p+\nu}{T_n}\hspace{-3pt}\displaystyle\sum_{t\in F_{n}}\log\left(1 + \dfrac{\xv_t^{l \top}\Lc\wv_{n}\,{\xv_{t}^l}}{\nu} \right)  +\beta \normop{\wv_{n}}_0   \\
    &+ \frac{\rho}{2}\normop{\Lc\wv_{n}-\Lb_{n}^l +\frac{1}{\rho}\Phib_{n}^l}_F^2 
     \\
     &+\frac{\rho}{2}\normop{\uv_{n}^l-\wv_{n}+\av^l\odot \hat{\wv}_{n-1} + \frac{1}{\rho}\muv_{n}^l}^2\\
      &+ \frac{\rho}{2}\normop{\mathfrak{d}\wv_{n}-\dv+\frac{1}{\rho}\zv_{n}^l}^2 \!\!+ \eta \tr\left(\Lc\wv_n\Vb_n^l\Vb_n^{l\top}\right).   
    \end{split}
\end{equation}

To address this challenging problem, we employ the majorization minimization (MM)  technique \cite{sun_majorization-minimization_2017} to minimize a surrogate majorization function of the original cost. First, using the inequality $\log(x)\leq x-1,\,\, \forall x>0$, we obtain:
\begin{align}
\label{eq:surr_1}
\begin{split}   
    \log\left(1 + \dfrac{\xv_t^{l \top}\Lc\wv_{n}\,{\xv_{t}^l}}{\nu} \right) &\leq \log\left(1 + \dfrac{\xv_t^{l \top}\Lc\wv^l_{n}\,{\xv_{t}^l}}{\nu} \right) 
    \\
    &\hspace{10pt}
    + \frac{\xv_t^{l \top}\Lc\wv_{n}\,{\xv_{t}^l}+\nu}{\xv_t^{l \top}\Lc\wv^l_{n}\,{\xv_{t}^l}+\nu} - 1 \\
    &=\left\langle \wv_{n}, \Lc^\ast \left( \frac{{\xv_{t}^l}\xv_t^{l \top }}{\xv_t^{l \top}\Lc\wv^l_{n}\,{\xv_{t}^l}+\nu}\right) \right\rangle 
    +\\
    &\hspace{132pt}\cte.
    \end{split}
\end{align}
Here, $\wv^l_{n}$ represents a fixed point, selected  as  the solution from the previous iteration.
We can also define a majorization function for the term $\normop{\Lc \wv_{n}}_F^2 + \normop{\mathfrak{d}\wv_n}^2 = \wv_{n}^ \top \Hb\wv_{n} $ as 
\begin{align}
\label{eq:surr_2}
\begin{split}
    \wv_{n}^ \top \Hb\wv_{n} 
    &\leq \wv_{n}^ \top \Hb\wv_{n} +
    (\wv_{n}-\wv_{n}^l)^\top (\zeta\Ib- \Hb ) (\wv_{n}-\wv_{n}^l) \\
    & = \zeta \normop{\wv_{n}-\wv_{n}^l}^2 + 2 \langle\wv_{n}, \Hb \wv_{n}^l \rangle + \cte,
    \end{split}
\end{align}
where $\Hb = \Lc^\ast\Lc +\mathfrak{d}^\ast \mathfrak{d}$ and $\zeta \geq \lambda_{\max}(\Hb) =  4p-2$ \cite{kumar_unified_2020}.
Using the proposed majorization  functions (upper-bounds) in \eqref{eq:surr_1} and \eqref{eq:surr_2} with $\zeta=4p-2$, the update step for $\wv_{n}$ would be simplified as follows:
\begin{equation}
\begin{split}  
    \wv_{n}^{l+1} &= \argminn_{\wv_{n}\geq \zerov} \,\,\,\frac{\rho(4p  - 1)}{2}\normop{\wv_{n} - \cv^l}^2 + \beta \normop{\wv_{n}}_0  
    \\
    &=\Id(\cv^l>\cv_{th})\odot \cv^l,
    \label{eq:w_update}
    \end{split}
\end{equation}
where $\cv_{th} = \sqrt{\frac{2\beta}{\rho(4 p -1)}}\onev$ and
\begin{equation}
\begin{array}{rl}
 \cv^l &\hspace{-8pt}=  \left(1-\frac{\rho}{\rho(4p-1)}\right) \wv_{n}^l -  \dfrac{ 1}{\rho(4p -1)} \left(\av^l + \bv^l \right), \\
 \av^l &\hspace{-8pt}= \Lc^*\left(\tilde{\Sb}^l 
  + \Phib_{n}^l + \rho \left(  \Lc\wv_n^l-\Lb_{n}^{l+1}\right) 
  +\eta \Vb_n^l\Vb_n^{l\top}
  \right), \\
  \bv^l &\hspace{-8pt}= -\muv_{n}^l - \rho\left(\uv_n^l  + \av^{l}\odot\hat{\wv}_{n-1}\right) +\mathfrak{d}^*\left(\zv_{n}^l - \rho\left(\dv - \mathfrak{d}\wv_{n}^l\right)\right),\\
  \tilde{\Sb}^l &\hspace{-8pt}= \frac{p+\nu}{T_n} \sum_{t\in F_{{n}}} \dfrac{ \xv_t\xv_t^\top}{\xv_t^\top \Lc \wv_{n}^l \xv_t + \nu}.
  \label{eq:w-update-a-b-c}
  \end{array}
\end{equation}
In the equation above,  $\Lc^\ast$ and $\mathfrak{d}^\ast$ represent  the
adjoints of the Laplacian and the degree operators, as defined in \cite{kumar_unified_2020}.

\subsection*{$\uv_{n}$-update step}
The subproblem associated with to the update step of $\uv_{n}$ admits a closed-form solution given by:
\begin{align}
\label{eq:u_update}
\begin{split} 
\uv_{n}^{l+1} &=\argminn_{\uv } 
    \frac{\rho}{2}\normop{\uv_{n}-\wv^{l+1}_{n}+\av^l\odot\hat{\wv}_{n-1}+\frac{1}{\rho}  \muv_{n}^l}^2 
    \\
    &\hspace{155pt}+ \,\alpha\normop{\uv_{n}}_1 \\
    &= \Sc_{\frac{\alpha}{\rho}}\left( \wv^{l+1}_{n} - \av^l\odot\hat{\wv}_{n-1} - \frac{1}{\rho} \muv_{n}^l\right),
\end{split}
\end{align}

where $\Sc$ denotes the soft-thresholding operator \cite{tibshirani_regression_1996}.

\subsection*{$\Xb_{n}$-update step}
The update step for \(\Xb_{n}\) is obtained by solving the following problem:
\begin{align}
\begin{split}  
\Xb_{n}^{l+1} &= \{\xv_t^{l+1}\vert_{t\in F_{n}}\} =\argminn_{\{\xv_t\vert_{t\in F_{n}}\}} 
    \, \displaystyle\sum_{t\in F_{n}} f_{\xv_t}(\xv_t) \\
    f_{\xv_t}(\xv_t) &=  \frac{1}{T_n\sigma_n^2} \normop{\yv_{t}-\mv_{t}\odot\xv_{t}}^2 
    \\
    &\hspace{50pt}
    +\frac{p+\nu}{T_n}\log\left(1 + \dfrac{\xv^\top_{t}\Lc\wv^{l+1}_{n}\,{\xv_{t}}}{\nu} \right). 
    \end{split}
\end{align}
This can be decomposed into smaller problems for each \(\xv_t\) as follows:
\begin{align}
\begin{split}  
\xv_t^{l+1} &=\argminn_{\xv_t } 
    f_{\xv_t}(\xv_t),\quad t\in F_n.
\end{split}
\end{align}
To find a closed-form solution to this problem, we   replace $f_{\xv_t}(\xv_t)$ with a majorization function as proposed by the following proposition.  
\begin{proposition}
    \label{prop:2}
    Let \(\tau \geq \frac{1}{\sigma_n^2} + \frac{p+\nu}{\xv^{l\top}_{t} \Lc \wv^{l+1}_{n} \xv_{t}^{l} + \nu} \lambda_{\max}(\Lc \wv^{l+1}_{n})\). Define the function 
    \begin{align}
        f_{\xv_t}^S(\xv_t, \xv_0) = \frac{\tau}{T_n} \normop{\xv_t - \xv_0 + \frac{\Qb_t \xv_0 - \cv_t}{\tau}}^2 + C(\xv_0),
    \end{align}
    where
    \begin{align}
    \label{eq:Q_c}
    \begin{split} 
        \Qb_t &= \frac{1}{\sigma_n^2} \Diag(\mv_t) + \frac{p+\nu}{\xv^{l\top}_{t} \Lc \wv^{l+1}_{n} \xv_{t}^{l} + \nu} \Lc \wv^{l+1}_{n}, \\
        \cv_t &= \frac{1}{\sigma_n^2} \yv_t.
            \end{split}
    \end{align}
    Here, \(\xv_0\) and \(C(\xv_0)\) are constants. Then, \(f^S_{\xv_t}(\xv_t, \xv_0)\) serves as a majorization function for \(f_{\xv_t}(\xv_t)\), satisfying the inequality 
    \(
        f_{\xv_t}(\xv_t) \leq f^S_{\xv_t}(\xv_t, \xv_0), \, \forall \xv_0.
    \)
\end{proposition}
\begin{proof}
    See appendix \ref{App:2}.
\end{proof}

By applying this proposition with \( \xv_0 = \xv_t^l \) in the context of a majorization-minimization framework, we formulate and solve the following problem for the update step of \( \xv_t \):
\begin{align}
\label{eq:x_update}
\begin{split}  
\xv_t^{l+1} &=\argminn_{\xv_t } 
    f_{\xv_t}^S(\xv_t, \xv_t^l)\\
    &= \xv^l_{t} - \frac{1}{\tau} \left( \Qb_t\xv^l_{t} - \frac{1}{\sigma_n^2}\yv_t\right).
    \end{split}
\end{align}
\subsection*{$\av$-update step}
The VAR parameters vector $\av$ can also be updated using the following closed-form solution:
\begin{align}
\label{eq:a_update}
\begin{split}  
\av^{l+1} &=\argminn_{\av\geq \zerov } 
    \,\frac{\rho}{2}\normop{\av\odot\hat{\wv}_{n-1}-\fv^l}^2 + \gamma \av^\top \onev 
    \\
    &= \Sc_{\frac{\gamma}{\rho\hat{\wv}_{n-1}^{\circ 2}} } \left((\fv^l)_+ \oslash \hat{\wv}_{n-1} \right) \odot \Id(\hat{\wv}_{n-1} >0),
    \end{split}
\end{align}
where $\fv^l = \wv^{l+1}_{n} -\uv^{l+1}_{n} -\frac{1}{\rho}  \muv_{n}^l $.

\subsection*{$\Vb_n$-update step}
Next, we have the update formula for $\Vb_n$ as follows:
\begin{align}
\label{eq:V_update}
    \Vb_n^{l+1}  &= \hspace{-5pt}\underset{ \Vb_n\in \Real^{p\times k},\Vb_n^\top\Vb_n =\Ib  }{\mathsf{argmin}} ~ \hspace{-5pt}\tr\left(\Lc\wv_n^{l+1}\Vb_n\Vb_n^\top\right) = \Qb_n^{l+1}[:,1:k].
\end{align}
Here, 
 $\Qb_n^{l+1}[:,1:k]$ denotes the set of eigenvectors of $\Lc\wv_n^{l+1}$ corresponding  to the first $k$ eigenvalues, sorted in ascending order. 

\subsection*{Update step for the dual variables}
Finally, we have the update step for the dual variables  as follows:
\begin{align}
\label{eq:dual_update}
\begin{split}
\Phib_{n}^{l+1} &= \Phib_{n}^{l} + \rho \left(\Lc\wv_{n}^{l+1} - \Lb_{n}^{l+1}\right),  \\
 \muv_{n}^{l+1} &=  \muv_{n}^{l} + \rho \left( \uv_{n}^{l+1} - \wv_{n}^{l+1} + \av^{l+1}\odot\hat{\wv}_{n-1} \right),  \\
\zv_{n}^{l+1} &= \zv_{n}^{l} + \rho \left(\mathfrak{d}\wv_{n}^{l+1} - \dv \right). 
\end{split}
\end{align}
The proposed method is  summarized in Algorithm \ref{algorithm:alg1}, with Theorem~\ref{thm:convergence} establishing the convergence.

\begin{theo}\label{thm:convergence}
The sequence of the augmented Lagrangian $\big\{ L_\rho \big(\Lb_{n}^l, \wv_{n}^l, \uv_{n}^l, \Xb_{n}^l, \av^l, \Vb_n^l, \Phib_{n}^l, \muv_{n}^l, \zv_{n}^l \big) \big\}$ generated by Algorithm~\ref{algorithm:alg1} converges for any sufficiently large \(\rho\). At the limit point, the equality constraints \(\Lb_{n} = \Lc \wv_{n}\), \(\uv_{n} = \wv_{n} - \av \odot \hat{\wv}_{n-1}\), and \(\mathfrak{d} \wv_{n} = \dv\) are also satisfied, i.e., 
\begin{align}
\begin{split}
&\lim_{l \to + \infty} \big\| \Lc\wv_{n}^{l} - \Lb_{n}^{l} \big\|_F = 0\\
&\lim_{l \to + \infty} \big\| \mathfrak{d}\wv_{n}^{l} - \dv \big\| = 0\\ 
&\lim_{l \to + \infty} \big\| \uv_{n}^{l} - \wv_{n}^{l} + \av^{l} \odot \hat{\wv}_{n-1} \big\| = 0.
\end{split}
\end{align}
\end{theo}
\begin{proof}
    See Appendix \ref{App:3}.
\end{proof}

{\centering
\begin{minipage}{\linewidth}
\begin{algorithm}[H]
\caption{Proposed algorithm 
for learning time-varying $k$-component graph from heavy-tailed data}
\vspace*{1ex}
\begin{algorithmic}[1]
\State \textbf{Input:}   The observation matrix $\Yb_{n}=[\yv_t\vert t\in F_{n}]$, the sampling mask $\Mb_{n}$ at  time frame $n$, and $\hat{\wv}_{n-1}$
\\ \textbf{Parameters:}   $k$, $\dv$, $\nu$, $\sigma_\epsilon$, $\gamma$, $\rho$
\State \textbf{Output:} The signal and the estimated graph weights at time frame $n$, i.e., $\Xb^l_n$ and $\wv^l_{n}$
\State \textbf{Initialization:} \textcolor{black}{$\wv_{n}^{0} = \hat{\wv}_{n-1}$,~ $\Phib_{n}^{0}= \zerov$, ~$\muv_{n}^{0}= \zerov$, ~$\zv_{n}^{0}= \zerov$, ~$\Xb^0_{n} = \Yb_{n}$,~ $\av^{0}$}
 \State Set $l=0$
\Repeat
\State   Update $\Lb_{n}^{l+1}$ using  \eqref{eq:Theta_update}.
\State  Update  $\wv_{n}^{l+1}$ via \eqref{eq:w_update}.
\State  Update  $\uv_{n}^{l+1}$ via \eqref{eq:u_update}.
\State  Update $\xv_t^{l+1}$ for $t\in F_{n}$ via \eqref{eq:x_update}.
\State   Update $\av^{l+1}$ using  \eqref{eq:a_update}.
\State   Update $\Vb_n^{l+1}$ using  \eqref{eq:V_update}.
\State Update the dual variables via \eqref{eq:dual_update}.
\State Set $l \leftarrow l+1$.
\Until {a stopping criterion is satisfied}
\State Set $\hat{\wv}_{n} = \wv^{l}_{n}$.
\end{algorithmic}
\label{algorithm:alg1}
\end{algorithm}
\end{minipage}
\par
}

{\color{black}
\subsection*{Computational Complexity}
The update step for \(\wv_n\) involves the computation of \(\tilde{\Sb}\), which has a complexity of \(\Oc(T_n p^2)\). Given this, the complexity of the closed-form solution in \eqref{eq:w_update} would be \(\Oc(p^2)\). This also holds for the update of \(\uv_n\) via \eqref{eq:u_update}. However, the update steps in \eqref{eq:Theta_update} and \eqref{eq:V_update} require eigenvalue decomposition of \(p \times p\) matrices, which is generally \(\Oc(p^3)\) complex. Hence, the overall complexity of the proposed algorithm is \(\Oc(p^3 + T_n p^2)\). While this may indicate that the proposed method may not be scalable to very large graphs, it is intrinsic to every graph-based clustering method that deals with the eigenvectors of the graph Laplacian.
}

\section{Numerical Results}
\label{sec: Simulation}
In this section, we present the numerical results of the proposed algorithm, comparing it to several state-of-the-art methods for time-varying graph learning across different scenarios. First, we evaluate the performance of our algorithm in learning a time-varying graph topology through a simulated experiment using synthetic data. Subsequently, we explore the application of this methodology in financial market analysis, focusing on data clustering and portfolio design. The results are detailed in the following two subsections.

\begin{figure*}[t]
\centering 
\begin{subfigure}[t]{0.2\textwidth}
   \includegraphics[trim=0 0 0 30,clip,width=\textwidth]{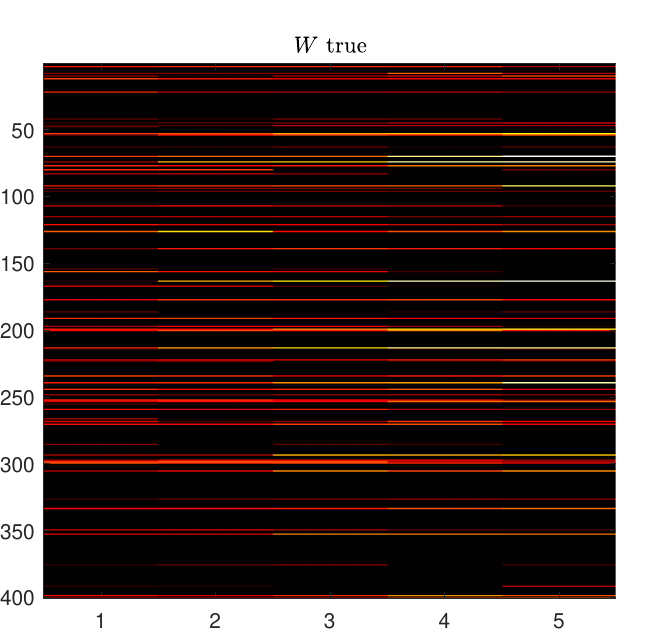}
     \caption{True weights
     }
\end{subfigure} 
\hspace{-10pt}
\begin{subfigure}[t]{0.2\textwidth}
   \includegraphics[trim=0 0 0 30,clip,width=\textwidth]{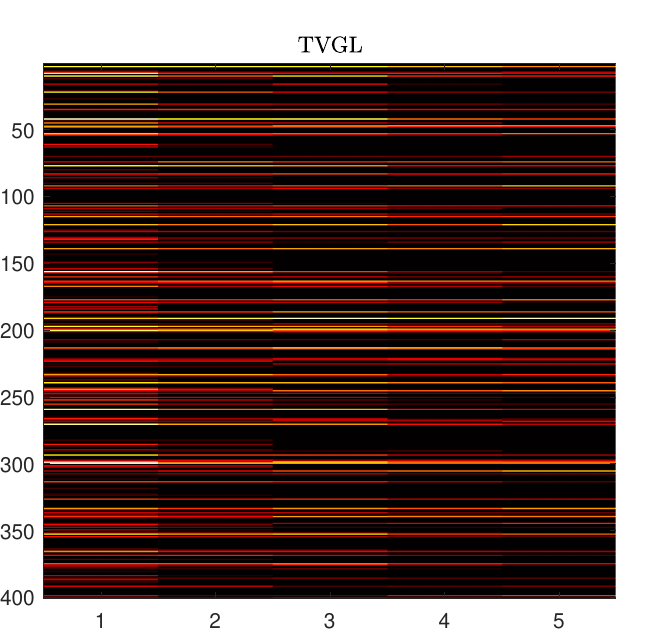}
     \caption{Proposed Method
     }
\end{subfigure} 
\hspace{-10pt}
\begin{subfigure}[t]{0.2\textwidth}
   \includegraphics[trim=0 0 0 30,clip,width=\textwidth]{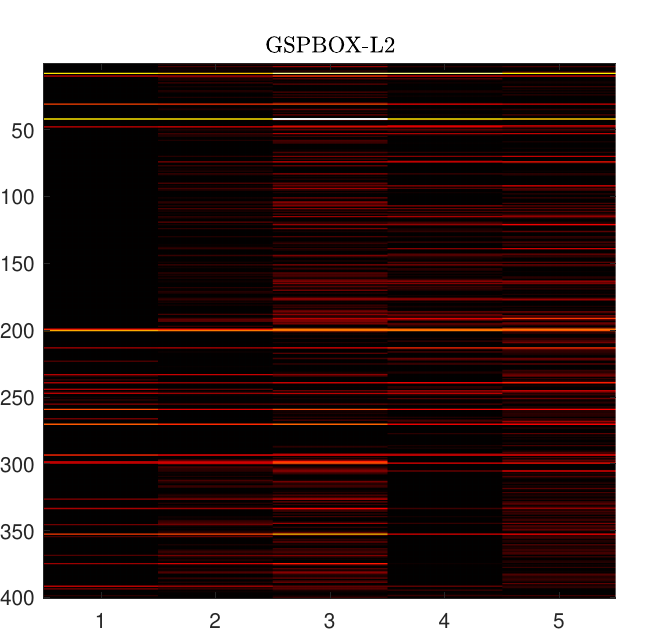}
     \caption{Kalofolias et al. \cite{kalofolias_learning_2017}
     }
\end{subfigure} 
\hspace{-10pt}
\begin{subfigure}[t]{0.2\textwidth}
 \includegraphics[trim=0 0 0 30,clip,width=\textwidth]{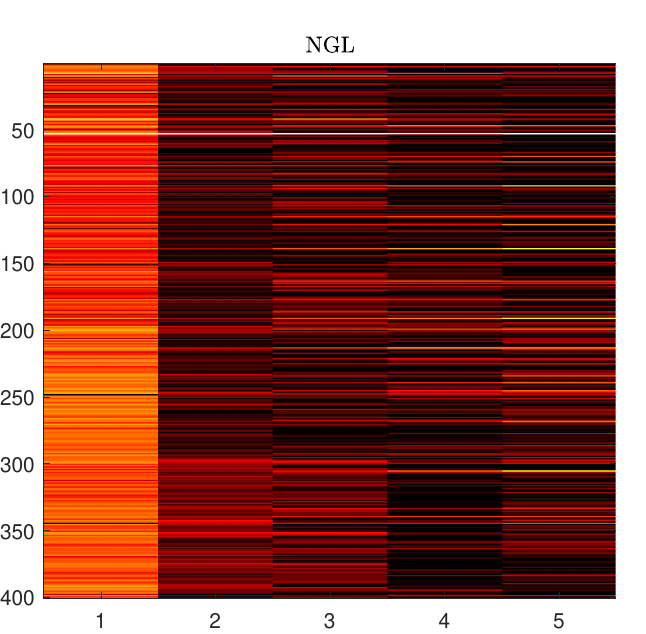}
     \caption{Cardoso et al. \cite{cardoso_learning_2020}
     }
\end{subfigure} 
\hspace{-10pt}
\begin{subfigure}[t]{0.2\textwidth}
   \includegraphics[trim=0 0 0 30,clip,width=\textwidth]{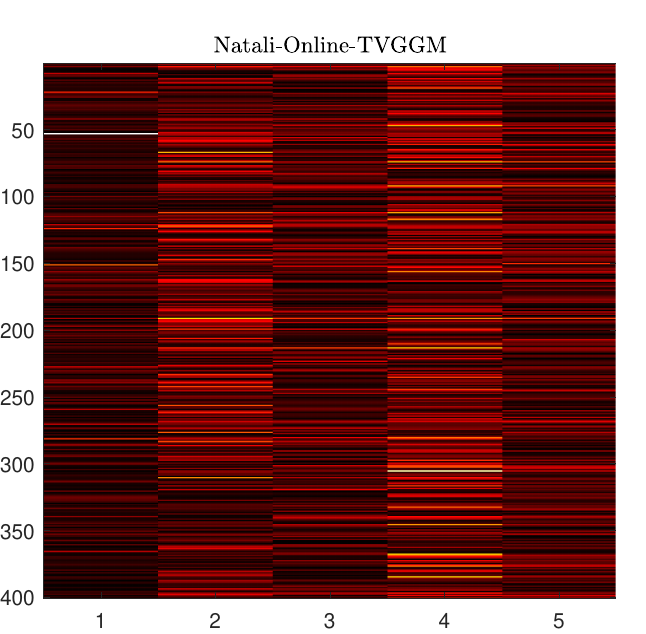}
     \caption{TV-GGM \cite{natali_learning_2022}
     }
\end{subfigure} 
 \\[3ex]
 \begin{subfigure}[t]{0.2\textwidth}
   \includegraphics[trim=0 0 0 30,clip,width=\textwidth]{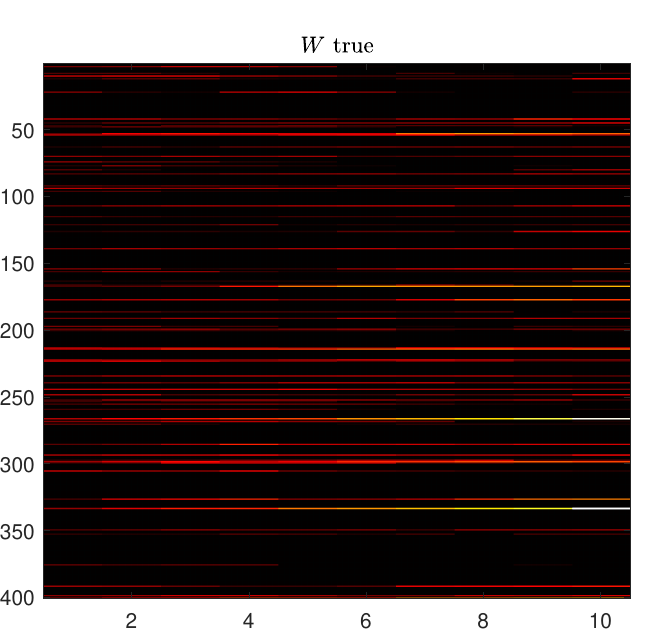}
     \caption{True weights
     }
\end{subfigure} 
\hspace{-10pt}
\begin{subfigure}[t]{0.2\textwidth}
   \includegraphics[trim=0 0 0 30,clip,width=\textwidth]{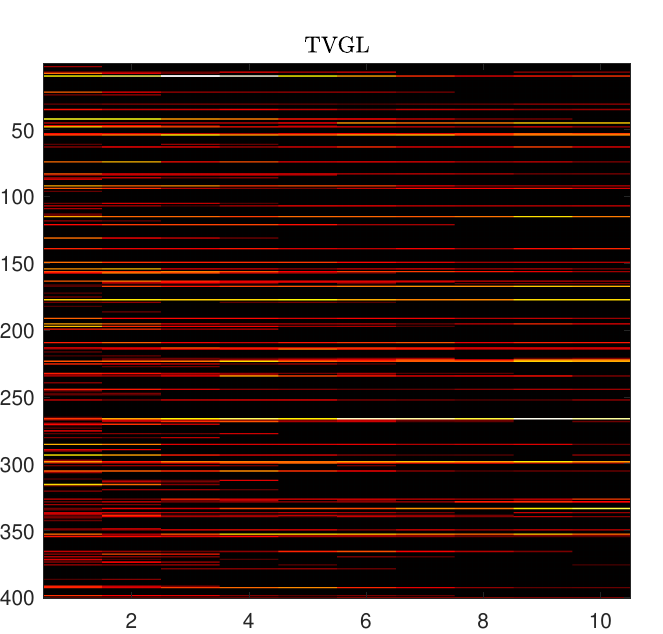}
     \caption{Proposed Method
     }
\end{subfigure}  
\hspace{-10pt}
\begin{subfigure}[t]{0.2\textwidth}
   \includegraphics[trim=0 0 0 30,clip,width=\textwidth]
   {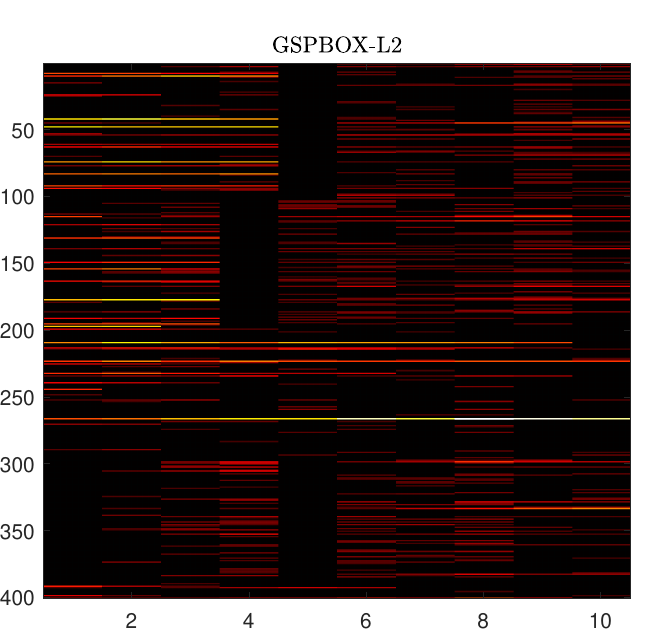}
     \caption{Kalofolias et al. \cite{kalofolias_learning_2017}
     }
\end{subfigure} 
\hspace{-10pt}
\begin{subfigure}[t]{0.2\textwidth}
 \includegraphics[trim=0 0 0 30,clip,width=\textwidth]{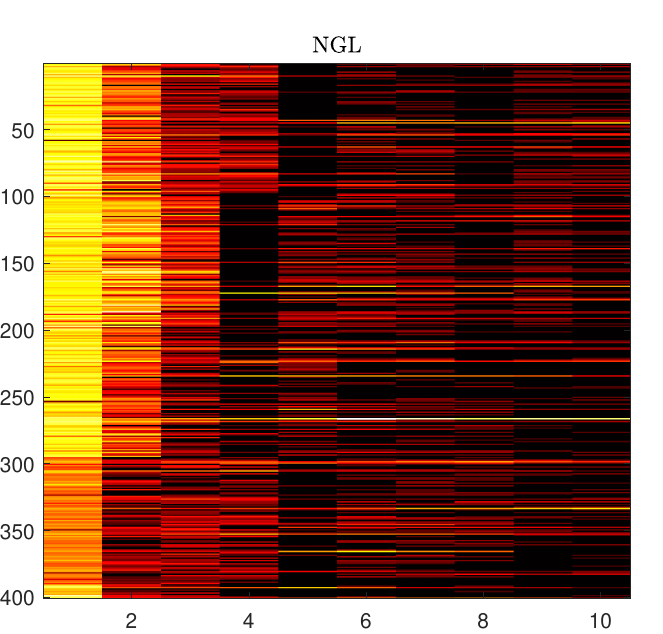}
     \caption{Cardoso et al. \cite{cardoso_learning_2020}
     }
\end{subfigure} 
\hspace{-10pt}
\begin{subfigure}[t]{0.2\textwidth}
   \includegraphics[trim=0 0 0 30,clip,width=\textwidth]{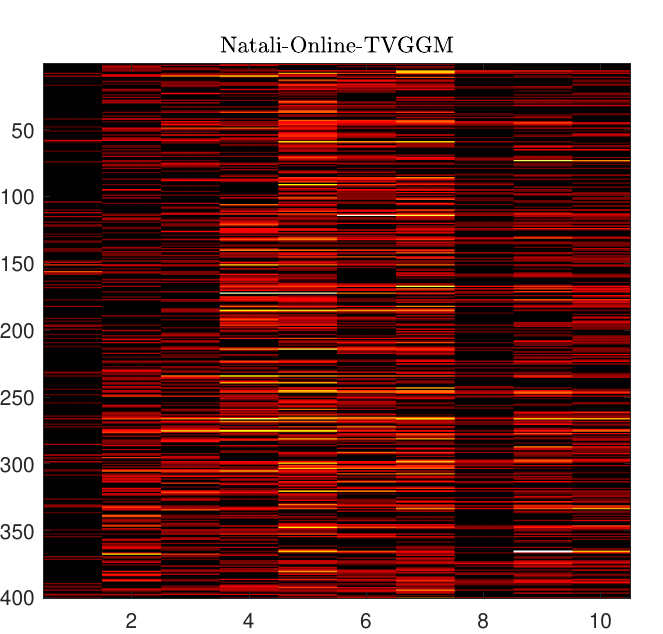}
     \caption{TV-GGM \cite{natali_learning_2022}
     }
\end{subfigure}
\caption{Visualization of the learned (weighted) adjacency matrices at different time intervals, with frames of length $T_n=200$ (top) and $T_n=100$ (bottom).
}
\label{fig_w_visual}
\end{figure*}

\begin{table*}[t]
 \scriptsize
 \centering
 \captionsetup{justification=centering}
 \caption {Performance of the graph learning methods for time-varying topology identification at different sampling rates \( \mathrm{SR} \) and fixed noise level \( \sigma_n = 0 \) (values are reported for the final frame with length \( T_n=200 \)).}
 \label{table:1}
 \begin{tabular}{m{1cm}||c|c||c|c||c|c} 
\multicolumn{1}{c||}{}  & \multicolumn{2}{c||}{$\mathrm{SR}=1$} &\multicolumn{2}{c||}{$\mathrm{SR}=0.8$} & \multicolumn{2}{c}{$\mathrm{SR}=0.6$} 
 \\
 \cline{2-7}
 & F-score  & RelErr 
 & F-score  & RelErr   
 & F-score  & RelErr \\
 \hline
 \hline
 \multicolumn{1}{c||}{Kalofolias et al. \cite{kalofolias_learning_2017}}  
 &  0.40 & 0.42 
 & 0.37 & 0.47
 & 0.33 & 0.52  
  \\
  \hline
 \multicolumn{1}{c||}{Cardoso et al. \cite{cardoso_learning_2020}}  
  &  0.36 & 0.37  
  & 0.31 & 0.40  
  & 0.29 & 0.41  
 \\
 \hline
 \multicolumn{1}{c||}{TV-SBM\cite{natali_learning_2022}}  
 &  0.38 & 0.48 
 & 0.35 & 0.49  
 & 0.32 & 0.50 
  \\
 \hline
 \multicolumn{1}{c||}{TV-GGM\cite{natali_online_2021-1} } 
 & 0.28 & 0.38 
 & 0.27 & 0.38 
 & 0.27 & 0.39 
 \\
 \hline
 \multicolumn{1}{c||}{Saboksayr et al. \cite{saboksayr_online_2021}}  
  &  0.28 & 0.49  
  & 0.28 & 0.49 
  & 0.28 & 0.51 
 \\
  \hline
 \multicolumn{1}{c||}{Proposed (Alg. \ref{algorithm:alg1})}   
  & \bf 0.60 & \bf 0.36 
  & \bf 0.56 & \bf 0.36 
  & \bf 0.51 & \bf 0.37
 \\
  \hline
 \end{tabular}
 \end{table*} 

\begin{table*}[t]
 \scriptsize
 \centering
 \captionsetup{justification=centering}
 \caption {Performance of the graph learning methods for time-varying topology identification at different noise levels  $\sigma_n$ and fixed sampling rate  $\mathrm{SR} = 1$ (values are reported for final frame with length $T_n=200$).}
 \label{table:1_noise}
 \begin{tabular}{m{1cm}||c|c||c|c||c|c} 
\multicolumn{1}{c||}{}  & \multicolumn{2}{c||}{$\sigma_n=0.1$} &\multicolumn{2}{c||}{$\sigma_n=0.3$} & \multicolumn{2}{c}{$\sigma_n=0.5$} 
 \\
 \cline{2-7}
 & F-score  & RelErr 
 & F-score  & RelErr   
 & F-score  & RelErr  \\
 \hline
 \hline
 \multicolumn{1}{c||}{Kalofolias et al. \cite{kalofolias_learning_2017}}  
 &  0.39 & 0.43 
 & 0.37 & 0.43 
 & 0.35 & 0.44  
  \\
  \hline
 \multicolumn{1}{c||}{Cardoso et al. \cite{cardoso_learning_2020}}  
  &  0.36 & 0.37  
  & 0.35 & 0.37  
  & 0.32 & 0.38 
 \\
 \hline
 \multicolumn{1}{c||}{TV-SBM\cite{natali_learning_2022}}  
 &  0.38 & 0.46 
 & 0.38 & 0.47  
 & 0.36 & 0.48  
  \\
 \hline
 \multicolumn{1}{c||}{TV-GGM\cite{natali_online_2021-1} } 
 &  0.28 & 0.38 
 & 0.28 & 0.38 
 & 0.28 & 0.38
 \\
 \hline
 \multicolumn{1}{c||}{Saboksayr et al. \cite{saboksayr_online_2021}}  
  &  0.29 & 0.47 
  & 0.28 & 0.47  
  & 0.28 & 0.49 
 \\
  \hline
 \multicolumn{1}{c||}{Proposed (Alg. \ref{algorithm:alg1})}   
  &  \bf 0.59 & \bf 0.36 
  & \bf 0.54 & \bf 0.36 
  & \bf 0.46 & \bf 0.37  
 \\
  \hline
 \end{tabular}
 \end{table*}

\subsection{Synthetic data}
For synthetic data generation, we consider $ p = 100 $ and $ T = 1000 $. We divide the $ T $ time-stamps into equal frames (windows) of length $ T_n $ (with no overlap), where the graph is assumed to be constant during each time frame. Let $ F_{n} $ denote the time indices in frame $ n $. Random samples of the signal in each frame are generated via
\(
\xv_t = \left(\Lb_{n}^\dagger\right)^{1/2} \nuv_t, \quad \nuv_t \sim \mathrm{St}(\zerov, \Ib), \quad t \in F_{n}
\),
where $ \Lb_{n}^\dagger $ represents the pseudo-inverse of the Laplacian matrix at time frame $ n $ and $ \mathrm{St} $ denotes the Student-\textit{t} distribution with zero mean and identity covariance matrix. To model the temporal variations of the graph weights, we use the equation 
\(
\wv_{n} = \left(\av \odot \wv_{n-1} + \epv_{n}\right)_+, \quad n \in \{1, \ldots, N\}.
\)
Here, $\av$ is sampled from an exponential distribution and $\epv_{n}$ is sampled from a normal distribution, where $N$ denotes the number of time frames. The initial values for the graph weights $\wv_{0}$ (and subsequently the Laplacian matrix $\Lb_0$) are sampled from the Stochastic Block Model, where the nodes are partitioned into 4 clusters (blocks) with random intra-cluster and inter-cluster edge weights.

We construct the original data matrix \( \Xb \) by concatenating the vectors \( \xv_t \) column-wise, covering the range from \( t=1 \) to \( t=T \). Following this, we normalize the data matrix such that each row is centered and scaled by its standard deviation. Next, we create a random binary sampling matrix \( \Mb \) defined by the sampling rate parameter \( \SR \), along with the observation noise matrix \( \Nb \), which consists of i.i.d. Gaussian random entries with zero mean and variance \( \sigma_n^2 \). The observation matrix is then formed as \( \Yb = \Mb \odot (\Xb + \Nb) \).

We then introduce the matrices \( \Yb \) and \( \Mb \) as inputs to the  time-varying graph learning algorithms.
For the parameters of our proposed method, we  choose \( \dv = \onev \) and \( \rho = 3 \). We also set \( \sigma_\epsilon = \exp(0.005T_n) \), \( \nu = 3 \), and \( \gamma = 0.01 \). 
We compare our method with several benchmark algorithms, including the TV-GGM \cite{natali_online_2021-1} and the TV-SBM \cite{natali_learning_2022} methods, for online time-varying graph learning under Gaussian graphical and smoothness-based models. Additionally, we include the time-varying graph learning method in \cite{cardoso_learning_2020}, the online graph learning algorithm by Saboksayr et al. \cite{saboksayr_online_2021}, and the time-varying version of the GSP Toolbox\footnote{\href{https://epfl-lts2.github.io/gspbox-html/}{https://epfl-lts2.github.io/gspbox-html/}} graph learning method by Kalofolias et al. \cite{kalofolias_learning_2017} in our comparison.

To evaluate the performance of these  algorithms in terms of learning accuracy, we utilize the relative error (RelErr) and the F-score criteria. Let \( \asto{\Lb_N} \in \mathbb{R}^{p \times p} \) be the ground-truth Laplacian at frame \( N \) (last frame), and \( \widehat{\Lb}_N \in \mathbb{R}^{p \times p} \) be the estimated one. The relative error and the F-score measures on the last frame of data are then defined as follows:
\begin{align}
\text{RelErr} = \dfrac{\|\asto{\Lb_N} - \widehat{\Lb}_N\|_F}{\|\asto{\Lb_N}\|_F}, \qquad
\text{F-score} = \dfrac{2 \, \text{TP}}{2 \, \text{TP} + \text{FP} + \text{FN}}. \nonumber
\end{align} 

In the above equations, TP, FP, and FN respectively denote the number of correctly identified connections in the original graph, the number of connections falsely identified in the estimated graph (not present in the original one), and the number of connections from the original graph that are missing in the estimated one. 

Visual representations of the time-varying graph weights across each time frame (for 400 edges) learned using different algorithms are shown in Fig. \ref{fig_w_visual}. In these figures, the horizontal axis represents the frame index $n$, and the vertical axis represents the index of the graph's edges. The top row shows the results for frame lengths of 
$T_n= 200$, while the bottom row shows the results for $T_n=100$. Table \ref{table:1} and Table \ref{table:1_noise} also provide the relative error and the F-score performance of the graphs at the final time frame (with $T_n = 200$)
for different signal sampling rates and noise levels, respectively. As observed in Fig. \ref{fig_w_visual}, the time-varying weights of the graph learned using the proposed method align more closely with the ground-truth graph weights. This is also evident in Tables \ref{table:1} and \ref{table:1_noise} in terms of relative error and  F-score measures.

\begin{figure*}[!t]
\centering 
\begin{subfigure}[t]{0.35\textwidth}
   \includegraphics[trim=0 80 0 60,clip,width=\textwidth]
   {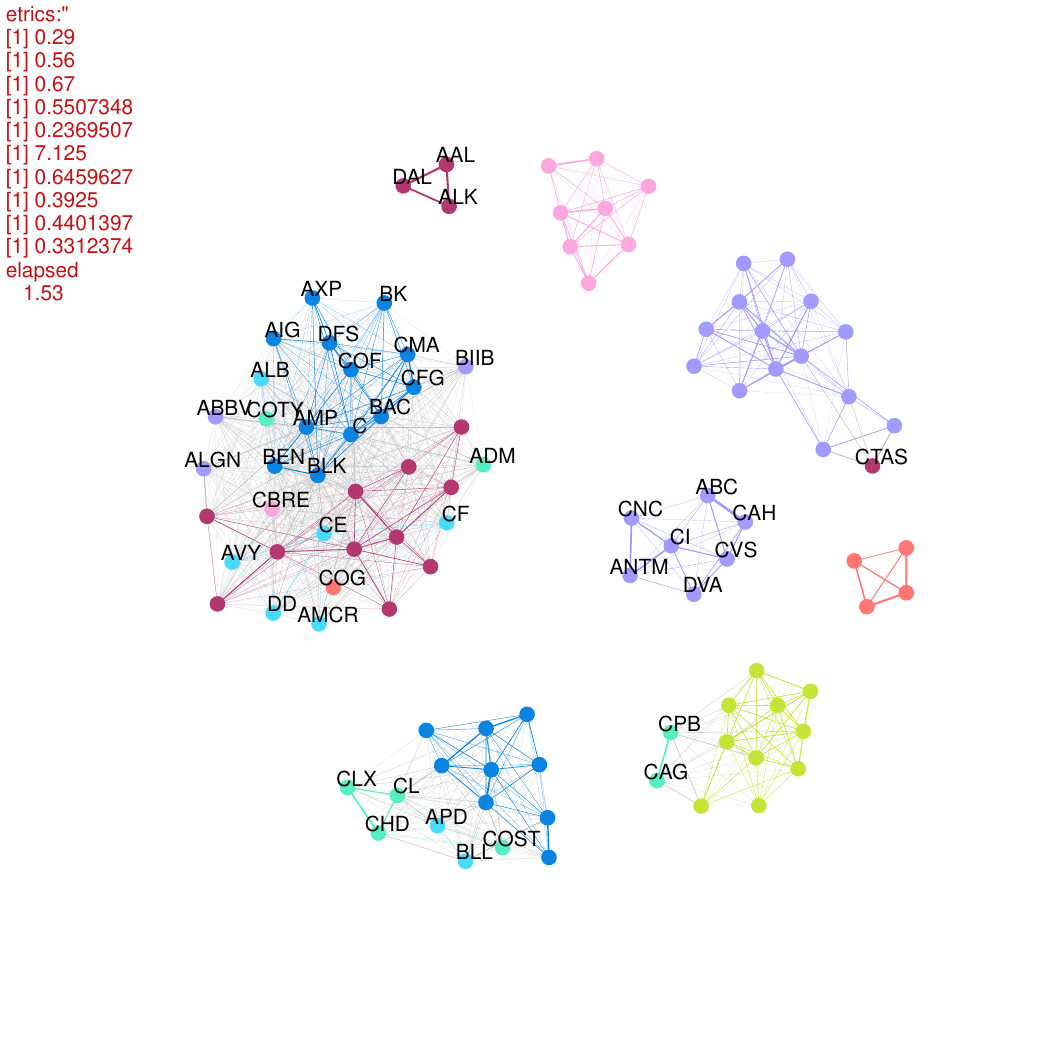}
     \caption{CLR  method \cite{nie_constrained_2016}
     }
\end{subfigure} 
\hspace{-20pt}
\begin{subfigure}[t]{0.35\textwidth}
   \includegraphics[trim=0 80 0 60,clip,width=\textwidth]
      {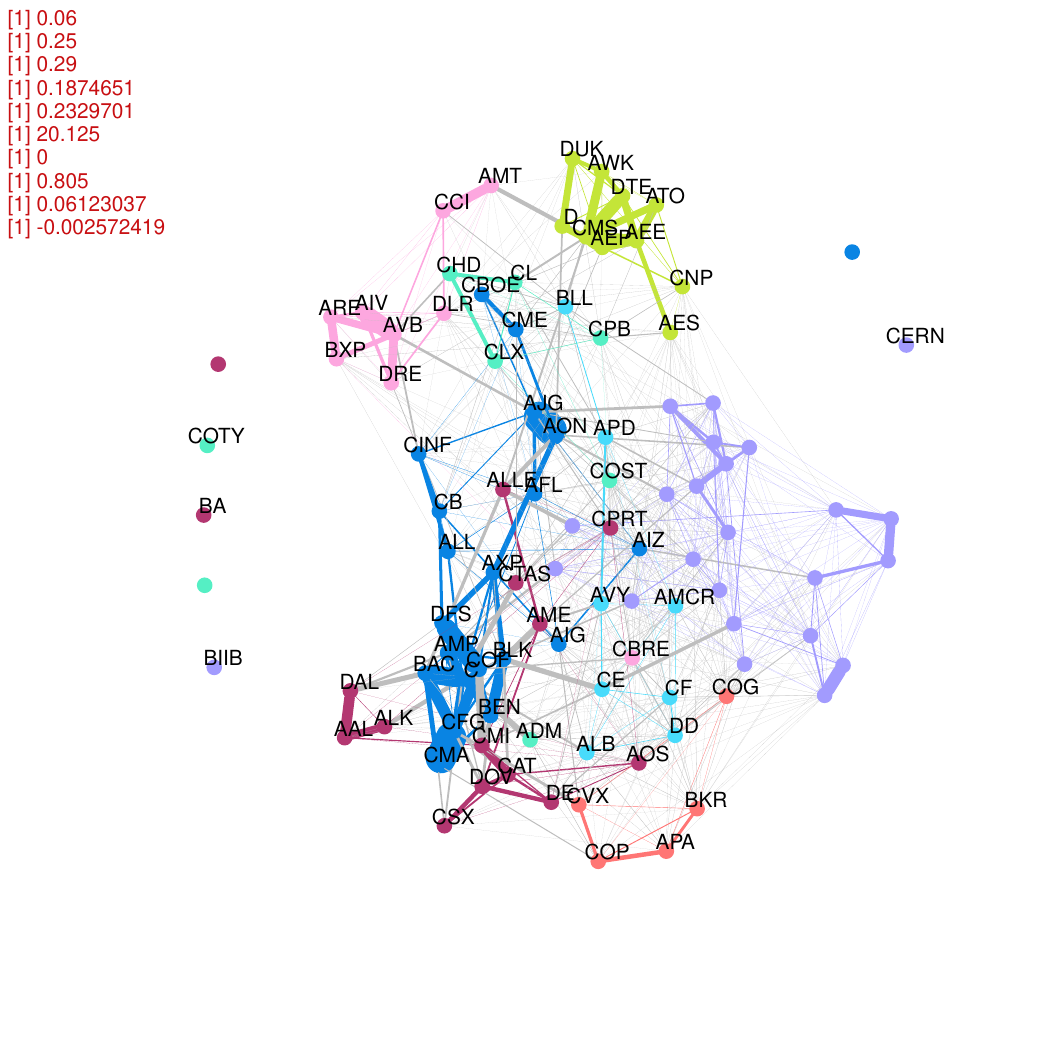}
     \caption{SGLA method \cite{kumar_unified_2020}
     }
\end{subfigure} 
\hspace{-20pt}
\begin{subfigure}[t]{0.35\textwidth}
   \includegraphics[trim=0 80 0 60,clip,width=\textwidth]
      {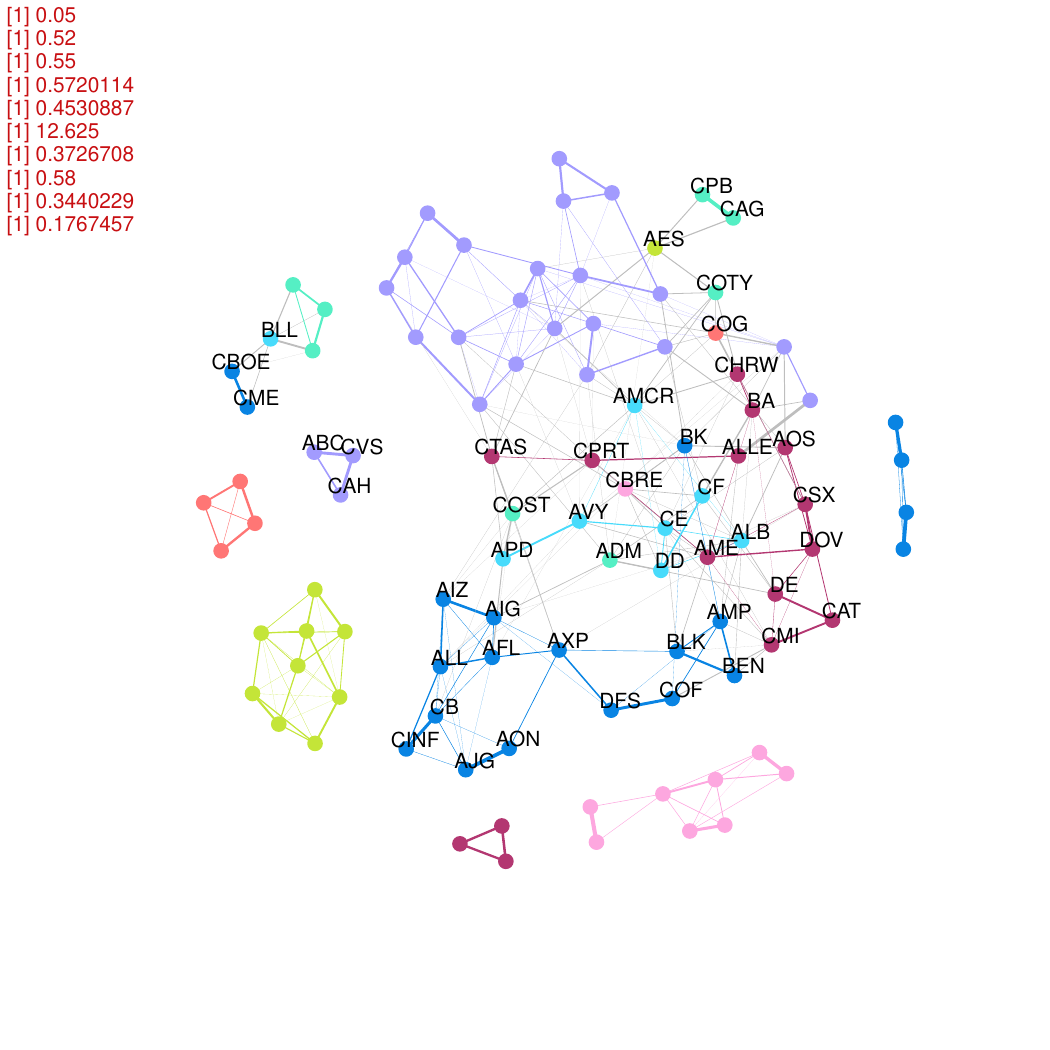}

     \caption{Fingraph method  \cite{de_miranda_cardoso_graphical_2021}
     }
\end{subfigure}
\hspace*{-20pt }
\\[10pt]
\hspace*{65pt}
\begin{subfigure}[t]{0.35\textwidth}
   \includegraphics[trim=0 80 0 60,clip,width=\textwidth]
{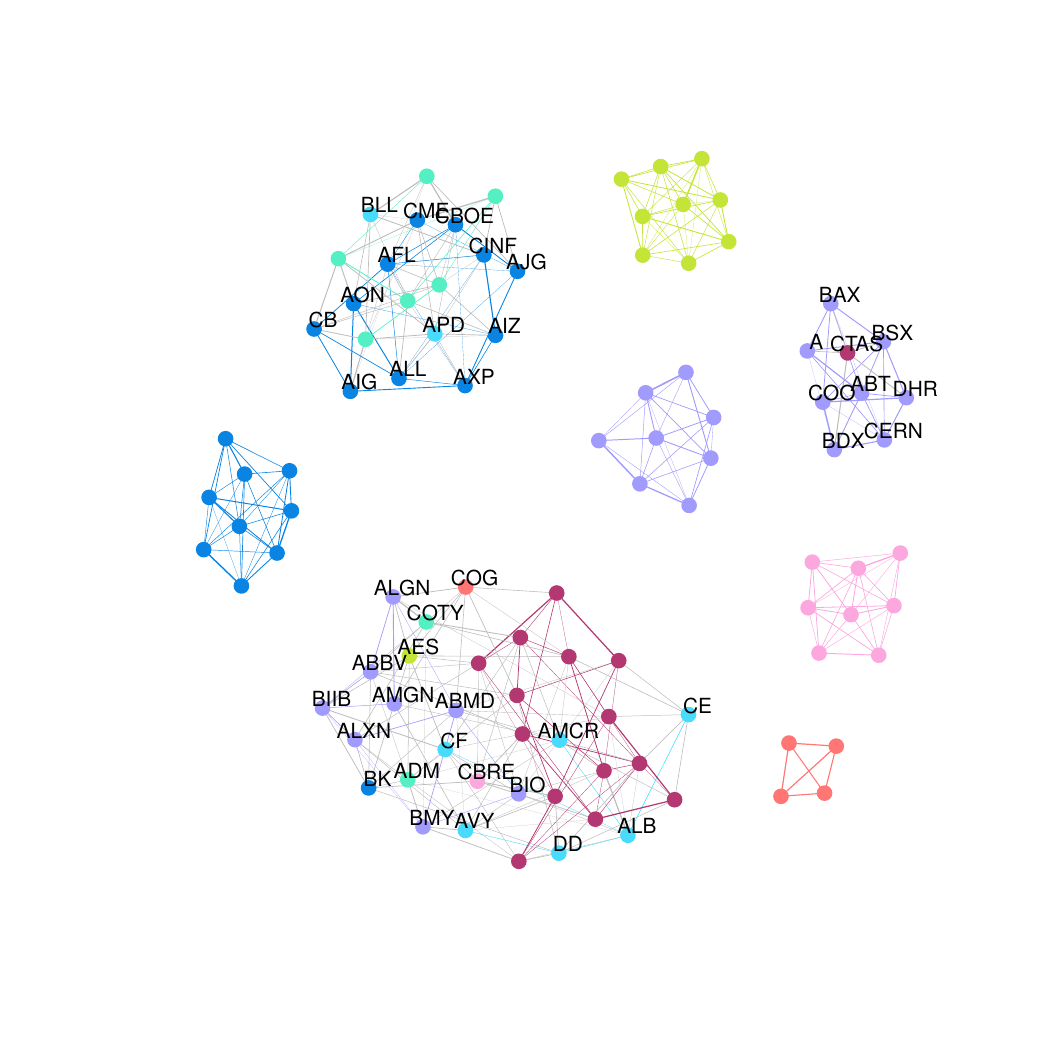}
     \caption{Javaheri et al.   \cite{javaheri_graph_2023} 
     }
\end{subfigure} 
\hspace{-10pt}
\begin{subfigure}[t]{0.35\textwidth}
   \includegraphics[trim=0 80 0 60,clip,width=\textwidth]      
   {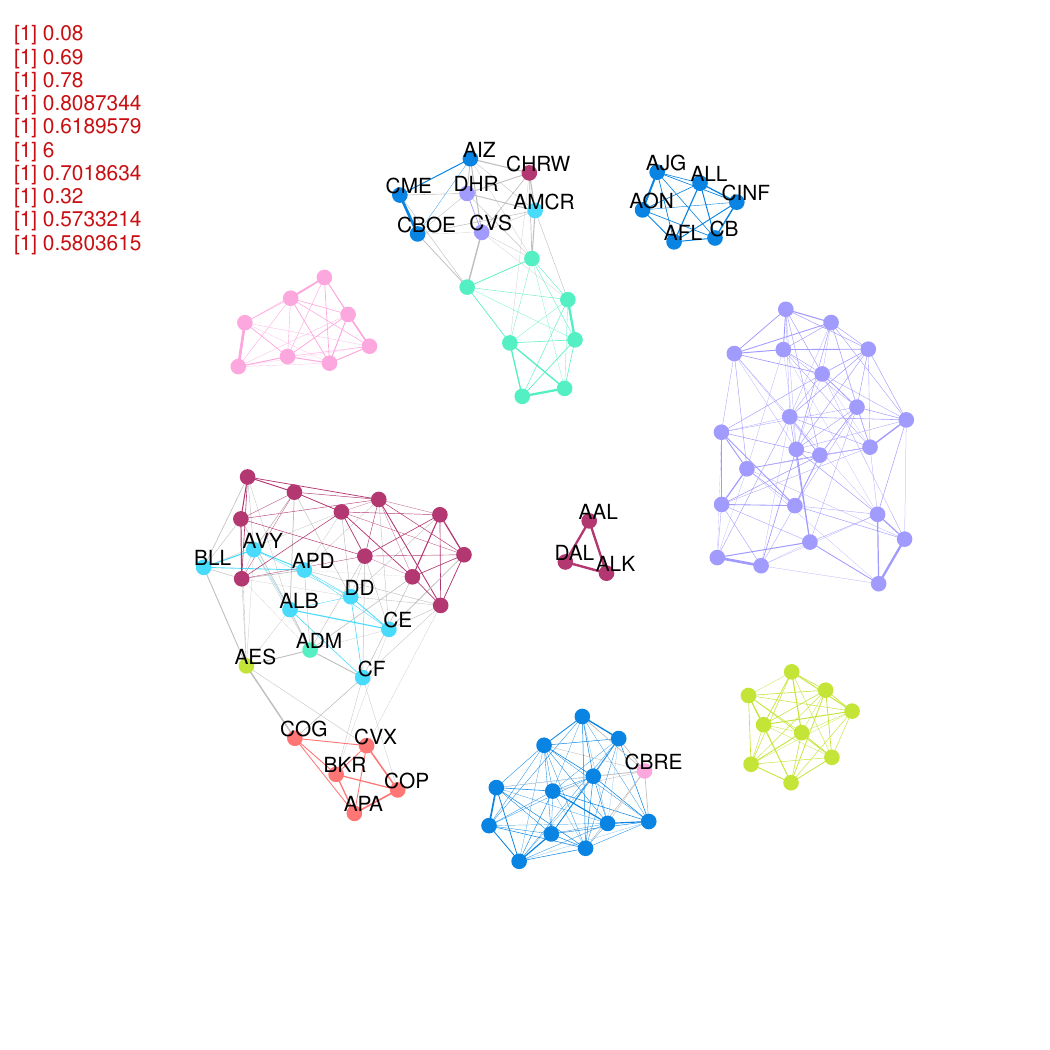}
     \caption{Proposed method (Alg. \ref{algorithm:alg1}).
     }
\end{subfigure} 
\hspace*{20pt}
\begin{subfigure}[t]{0.11\textwidth}
   \includegraphics[trim=0 0 0 0,clip,width=\textwidth]{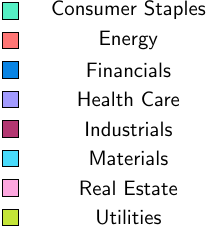}
\end{subfigure} 
\caption{The \( k \)-component graphs learned from financial data corresponding to the log-returns of 100 stocks in the S\&P 500 index (comprising \( k = 8 \) sectors). The graphs are shown for the last data frame (with length $T_n = 200$).
}
\label{fig_fingraph_orig}
\end{figure*}

\subsection{Real Data}
In this section, we utilize real-world data from financial markets, particularly the log-returns of the stocks in the S\&P500 index. For this experiment, we select a subset of 100 stocks, categorized into $k=8$ sectors (clusters), including Utilities", "Real State", "Materials", Industrials", Health Care", "Financials", "Energy", and "Consumer Staples". The ground-truth labels of the sectors are determined by the GICS classification standard\footnote{\href{https://www.msci.com/our-solutions/indexes/gics}{https://www.msci.com/our-solutions/indexes/gics}}. We subsequently compute the log-return of these stocks over a 1000-day period spanning from January 2016 to January 2020. Following this, we construct the matrix $ \Xb \in \mathbb{R}^{p \times T} $ with $ p=100 $ (number of stocks) and $ T=1000 $ (number of days). We segment the time domain (days) into  frames of length $ T_n $.

We first partition data into non-overlapping frames of length $T_n = 200$, resulting in $N=5$ data frames. 
Next, we evaluate the effectiveness of our method for two different applications: stock classification or clustering and portfolio design. For the clustering task, we employ several criteria to evaluate performance, including accuracy (ACC), purity \cite{everitt_cluster_2011}, 
the modularity (MOD) \cite{newman_modularity_2006},
and the adjusted Rand index (ARI) \cite{rand_objective_1971}. Both accuracy and purity are calculated by determining the ratio of the true-positive (correctly classified) labels (TP) to $p$. However, there is a distinction between these two metrics: in accuracy, we consider the best ordering of the labels assigned to the inferred clusters (among all $k!$ permutations), whereas in purity, the label of each cluster is assumed to be the ground-truth label of the majority of the nodes in that cluster.
The modularity is also a measure that evaluates how disjoint the nodes with different labels are (the higher the value, the more disjoint they are).
The ARI is another metric used to evaluate the similarity between the true labels and the cluster labels. Here, the parameter $ \nu $ in the proposed method is obtained by fitting a multivariate Student-\textit{t} distribution to the data (using the \texttt{fitHeavyTail} R-package\footnote{\href{https://cran.r-project.org/package=fitHeavyTail}{https://CRAN.R-project.org/package=fitHeavyTail}}).

\begin{table}[!b]
 \scriptsize
\centering
 \caption{Clustering performance of the graphs shown in Fig. \ref{fig_fingraph_orig}. The values are reported for the last data frame (with   $T_n=200$).}
 \label{table:3}
 \begin{tabular}{m{1cm}||c|c|c|c|c} 
     & ACC  & Purity & MOD & ARI & Time (s)  \\
     \hline
     
     \multicolumn{1}{c||}{CLR  \cite{nie_constrained_2016}} & 
      0.56 & 
      0.67 &  
      0.23 &
      0.33 &
      1.53\\
     \hline
     \multicolumn{1}{c||}{SGLA  \cite{kumar_unified_2020} } & 
      0.29 & 
      0.31 & 
      0.24 &
     0.02 &
    \bf 1.21 \\
     \hline
     \multicolumn{1}{c||}{Fingraph \cite{de_miranda_cardoso_graphical_2021}  }  & 
      0.52 & 
      0.55 & 
      0.43 &
    0.18 &
    23.95\\
     \hline
     \multicolumn{1}{c||}{Javaheri et al.  \cite{javaheri_graph_2023}  }  & 
    0.58 & 
    0.71 & 
    0.45 &
    0.34 &
    2.72 \\
        \hline
     \multicolumn{1}{c||}{  \multirow{1}{*}Proposed (Alg. \ref{algorithm:alg1})}  
     &  \bf 0.69 
     &  \bf 0.78
     & \bf 0.62
     &  \bf 0.58
     & 2.14 \\
     \hline
 \end{tabular}
 \end{table}

Fig. \ref{fig_fingraph_orig} illustrates the graphs learned from the last frame of data ($n=N=5$), using the proposed method, along with several existing benchmark algorithms for multi-component graph learning (which are static methods).
The colors of the nodes (stocks) indicate their respective ground-truth clusters (sector indices), while labels adjacent to certain nodes denote the misclassified stocks. These algorithms include the constrained Laplacian rank (CLR) method \cite{nie_constrained_2016}, the SGLA method\footnote{\href{https://CRAN.R-project.org/package=spectralGraphTopology}{https://CRAN.R-project.org/package=spectralGraphTopology}} \cite{kumar_unified_2020}, the Fingraph algorithm\footnote{\href{https://CRAN.R-project.org/package=fingraph}{https://CRAN.R-project.org/package=fingraph}} \cite{de_miranda_cardoso_graphical_2021}, and the method proposed by Javaheri et al. \cite{javaheri_graph_2023} for balanced clustering. Notably, the latter two methods are tailored for heavy-tailed data. For these benchmark algorithms, which are all designed for offline static graph learning, 
we provide 
the static graph learned from each data frame as the initial guess for graph learning in the next frame, i.e., $\wv^0_{n+1} = \hat{\wv}_{n}$.

\begin{figure*}[!t]
\centering 
\begin{subfigure}[t]{0.32\textwidth}
\includegraphics[trim=0 80 0 60,clip,width=\textwidth]
{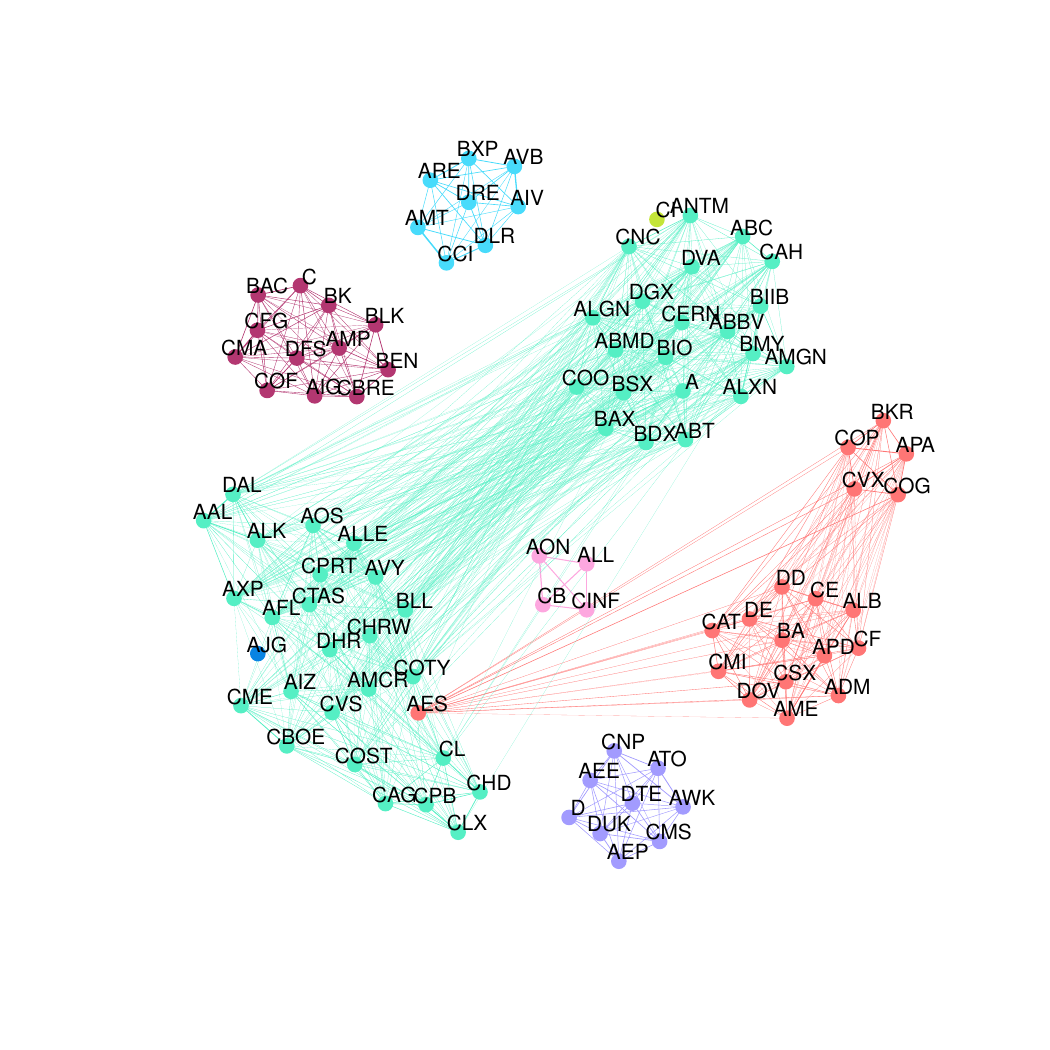}
     \caption{Graph learned  from the 1st frame ($T_n = 200$).
     }
\end{subfigure} 
\begin{subfigure}[t]{0.32\textwidth}
\includegraphics[trim=0 80 0 60,clip,width=\textwidth]
  {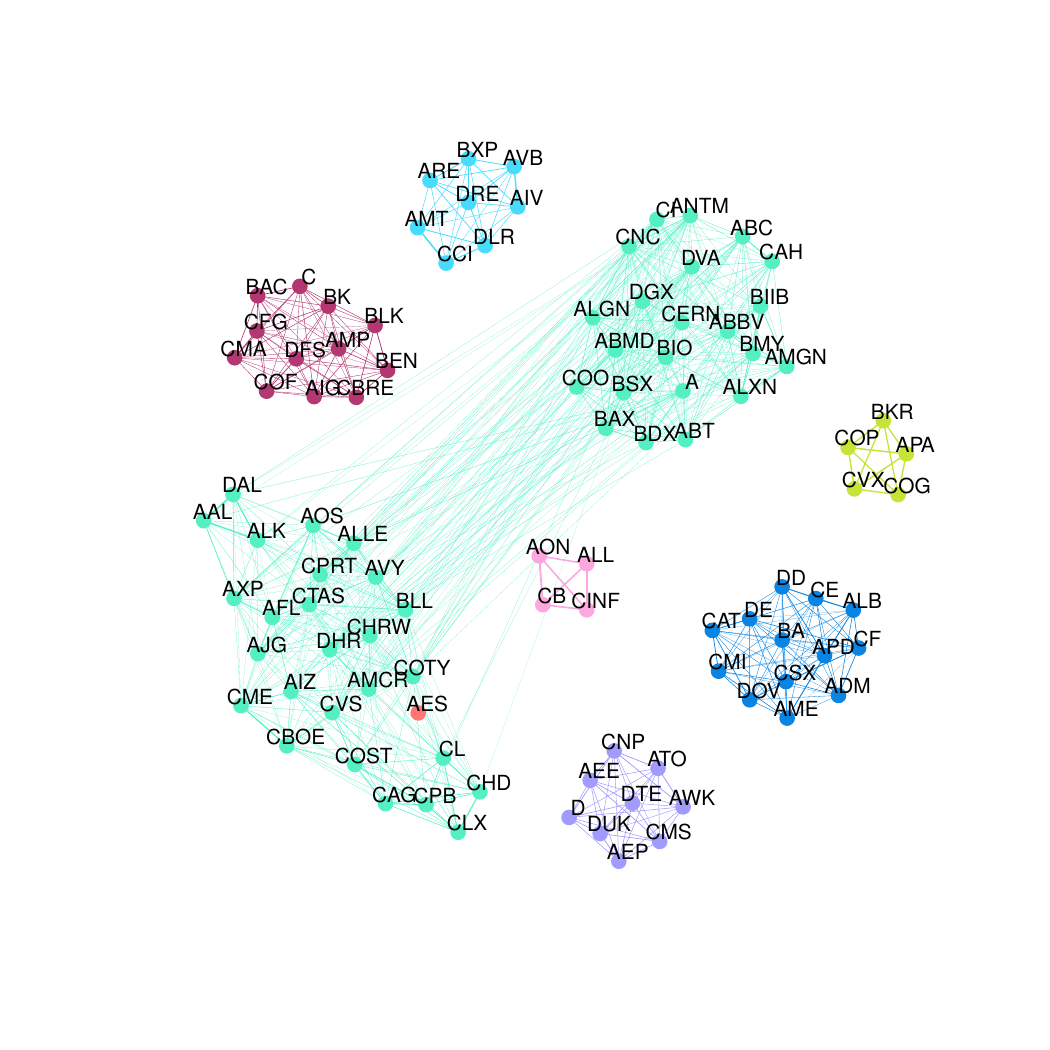}
     \caption{ Graph  learned from the 2nd frame ($T_n = 200$).
     }
\end{subfigure} 
\begin{subfigure}[t]{0.32\textwidth}
 \includegraphics[trim=0 80 0 60,clip,width=\textwidth]
{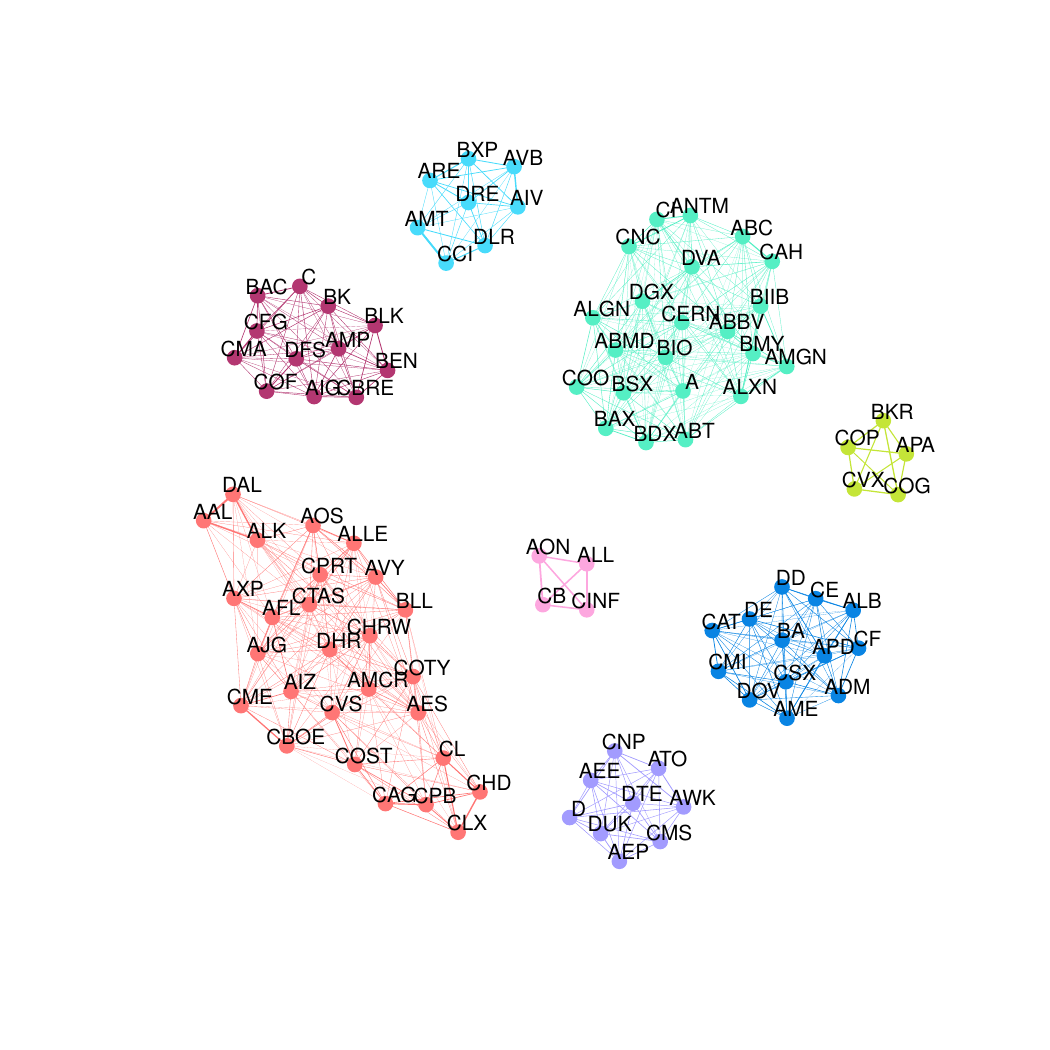}
     \caption{Graph  learned from the 3rd frame ($T_n = 200$).
     }
\end{subfigure} 

\caption{The time evolution of the graphs learned from  S\&P500 data via the proposed method for the frame length of $T_n = 200$. Colors represent the inferred clusters ($k=8$). 
}
\label{fig_tvgplot}
\end{figure*}

\begin{figure}[!t]
\centering 
\begin{subfigure}[t]{0.3\textwidth}
   \includegraphics[trim=0 0 0 0,clip,width=\textwidth]  {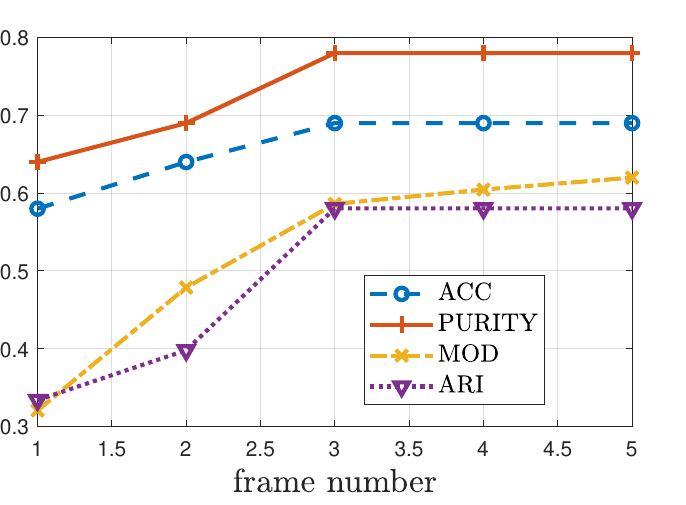}
\end{subfigure} 
\caption{Performance  of the proposed method 
for clustering SP500 data versus the frame number. 
}
\label{fig:acc_vec}
\end{figure}

\begin{table}[!b]
 \scriptsize
\centering
 \caption{Clustering performance of the proposed method for different choices of the frame length $T_n$. 
 }
 \label{table:diff_Tn}
 \begin{tabular}{m{1cm}||c|c|c|c|c} 
     & ACC  & Purity & MOD & ARI  & Delay (days) \\
     \hline
    \multicolumn{1}{c||}{  \multirow{1}{*} {$T_n = 100$}}  
     &  0.68 
     &  0.72
     &  0.59
     &  0.53
     &  \bf 100\\
     \hline
     \multicolumn{1}{c||}{  \multirow{1}{*} {$T_n = 200$}}  
     &  0.69 
     &  0.78
     & \bf 0.62
     &  0.58
     &  200\\
     \hline
     \multicolumn{1}{c||}{  \multirow{1}{*} {$T_n = 500$}} 
     &  0.71 
     &  0.84 
     & 0.60
     &  0.67
     &  500\\
     \hline
     \multicolumn{1}{c||}{  \multirow{1}{*}
     {\color{black} 
     $T_n = 1000$}}   
      &  {\color{black}\bf  0.79}
     & {\color{black}\bf  0.89}
     &  0.59 
     & {\color{black}\bf 0.75}
     &  1000\\
     \hline
  
     \hline
 \end{tabular}
 \end{table}

 The clustering performance corresponding to these methods as well as the runtime of the algorithms is summarized in Table \ref{table:3}. The results indicate that the 
 time-varying 
 graph learned using the proposed algorithm (Alg. \ref{algorithm:alg1}) exhibits better   clustering performance by taking advantage of the graph temporal variations, compared to the static methods for graph learning.
 As expected, the performance of the proposed time-varying clustering  method is improved through time-evolution as shown in Fig. \ref{fig_tvgplot}.
 This is also visually depicted in Fig. \ref{fig:acc_vec}, where the clustering performance is plotted against the frame number, illustrating its improvement over time.
We then evaluate the performance of our proposed method for  different choices of the frame length, namely $T_n = 200$, $T_n=500$, and $T_n = 1000$. The results are given in Table \ref{table:diff_Tn}, where the last column represents the number of days we have to wait to update the graph learned form financial data.
As illustrated in the table, the accuracy of the clustering improves  as the frame length increases, due to the availability of more data samples and more accurate statistical estimates. However, this also requires access to more data and introduces larger delay. Therefore, depending on the application, a trade-off between delay and accuracy should be considered.

We then compare our method with benchmark algorithms for time-varying graph learning, including the method proposed by Natali et al. (TV-GGM) \cite{natali_online_2021-1} for online time-varying graph topology identification under the Gaussian graphical model, the time-varying graph learning method proposed by Cardoso et al. \cite{cardoso_learning_2020}, the online time-varying graph topology inference method by Saboksayr et al. \cite{saboksayr_online_2021}, and the time-varying graph learning algorithm by Kalofolias et al. \cite{kalofolias_learning_2017}. These methods are not applicable for learning $k$-component graphs (based on which we can cluster the data). Therefore, we utilize the spectral clustering approach \cite{ng_spectral_2001} for node classification. Specifically, we apply $k$-means clustering \cite{macqueen_methods_1967} on the set of eigenvectors of the inferred Laplacian matrix corresponding to the $k$ smallest eigenvalues. Table \ref{table:clust_connect} presents the performance comparison of the proposed method with benchmark time-varying graph learning algorithms for clustering financial (S\&P 500) data. Here, the frame length is fixed at $T_n = 200$ for all methods, and the reported values represent the clustering performance on the last frame of the data ($n=N$). It is evident that the proposed method outperforms the others significantly, attributed to its capability to learn $k$-component graphs from heavy-tailed data.

\begin{table}[!t]
\scriptsize
\centering
 \caption{Clustering performance comparison between the proposed method and the benchmark algorithms for learning time-varying graphs, where $k$-means is employed for spectral clustering (the frame length is $T_n = 200$). 
 }
 \label{table:clust_connect}
 \begin{tabular}{m{1cm}||c|c|c|c|c} 
     & ACC  & Purity & MOD & ARI  & Time (s)  \\
     \hline
     
     \multicolumn{1}{c||}{Proposed (Alg. \ref{algorithm:alg1})
     }  & 
      \bf 0.69 & 
      \bf0.78 &
      \bf 0.62 &  
     \bf 0.58 &
     2.14 \\
     \hline
     \multicolumn{1}{c||}{Kalofolias et al.  \cite{kalofolias_learning_2017} }  & 
      0.47 & 
      0.49 & 
      0.59 &
      0.16 &
     \bf 0.18 \\
     \hline
     \multicolumn{1}{c||}{Cardoso et. al \cite{cardoso_learning_2020} }  & 
      0.57 & 
    0.67 & 
    \textcolor{black}{  0.36} &
     0.41 &
    1.45
    \\
    \hline
     \multicolumn{1}{c||}{  \multirow{1}{*}TV-GGM \cite{natali_online_2021-1}}  
      &  0.29 
     &  0.31
     &  0
     &  0.08
     & 6.76 \\
     \hline
     \multicolumn{1}{c||}{  \multirow{1}{*}TV-SBM \cite{natali_online_2021-1}}  
      &  0.61 
     &  0.63
     &  0.58
     & 0.36
     & 522.58 \\
     \hline
     
     \multicolumn{1}{c||}{  \multirow{1}{*}Saboksayr et al. \cite{saboksayr_online_2021}} 
     & 0.41 
     &  0.44 
     &  0.44
     & 0.12
     & 2.26\\
     \hline
 \end{tabular}
 \end{table}

Next, we explore the application of our time-varying graph learning method for portfolio design. We consider a dynamic portfolio design strategy based on maximizing the ratio of the portfolio return over the portfolio graph smoothness. Let $\Lb_{n}$ denote the time-varying graph learned by the proposed method (Algorithm \ref{algorithm:alg1}) given the $n$-th data frame. Let $\uv_n$ denote the portfolio weights at time frame $n$ and assume the expected value (mean) and the covariance matrix of the stock returns for the $n$-th data frame are given as $\hat{\muv}_n$ and $\hat{\Sigmab}_n$. Here, we use Student-\textit{t} robust estimators for $\hat{\muv}_n$ and $\hat{\Sigmab}_n$ (using the \texttt{fitHeavyTail} R-package similar to the previous part). The Maximum Sharpe Ratio Portfolio (MSRP) design scheme aims at maximizing the Sharpe ratio (S) \cite{sharpe_mutual_1966} defined as the ratio of the portfolio expected return over the volatility, i.e., $\text{S} = \frac{\hat{\muv}_n^\top \uv_n}{\uv_n^\top \hat{\Sigmab}_n \uv_n}$. Using the Schaible transform \cite{schaible_parameter-free_1974}, the dynamic MSRP optimization problem can be reformulated as:
\begin{align}
\label{eq:MSRP}
    \uv_{\text{MSRP},n}^\star = \argminn_{\uv_{n}\geq \zerov,\, \hat{\muv}_{n}^\top \uv_{n} = 1} \uv_{n}^\top \hat{\Sigmab}_{n} \uv_{n} 
\end{align}
where $\uv_{\text{MSRP},n}^\star$ denotes the optimal MSRP weights at time frame $n$. Inspired by the notion of the Sharpe ratio, we propose the Graph-based Ratio (GR) as 
\begin{align*}
    \text{GR} \triangleq \frac{\hat{\muv}_n^\top \uv_n}{\sqrt{\uv_n^\top \Lb_n \uv_n}}.
\end{align*}
To maximize this ratio, we design a portfolio referred to as the Maximum Time-Varying Graph Ratio Portfolio (MTVGRP). The problem for this portfolio design is stated as follows:
\begin{align}
\label{eq:MGSP}
    \uv_{\text{MTVGRP},n}^\star = \argminn_{\uv_{n}\geq \zerov,\, \hat{\muv}_{n}^\top \uv_{n} = 1} \uv_{n}^\top \Lb_{n} \uv_{n} 
\end{align}

Consider the returns of 50 randomly chosen stocks from the S\&P500 index over the period 2010-12-01 to 2018-12-01. From this dataset, we select 100 different subsets (time intervals), each of length $T=504$ days, with different starting time indices. Each dataset is partitioned into frames of length $T_n=200$ days with 180 days overlap. We then design dynamic portfolios (with re-optimization frequency of 20 days) based on our proposed MTVGRP scheme, and evaluate the performance using the portfolio backtest\footnote{\hyperlink{https://CRAN.R-project.org/package=portfolioBacktest }{https://CRAN.R-project.org/package=portfolioBacktest}} package in R. We compare our design scheme with the dynamic MSRP portfolio and the Equally Weighted Portfolio (EWP). The weights for the MSRP are obtained by solving \eqref{eq:MSRP} and the EWP weights are given as $\uv^\star_{\text{EWP},n} = 1/p$. Table \ref{table:backtest} shows the results of this backtest in terms of different performance criteria, including the Sharpe ratio, the annualized return, the annualized volatility, and the maximum drawdown. From this table, it is implied that the proposed method delivers better performance through the (dynamic) time-varying graph-based portfolio design scheme. This is also evident in Fig. \ref{fig:sr_barchart} where the Sharpe ratio and the maximum drawdown of the portfolios in Table \ref{table:backtest} are compared with the market index.

\begin{figure}[!t]
\centering 
\hspace*{30pt}\begin{subfigure}[t]{0.45\textwidth}
   \includegraphics[trim=0 0 0 0,clip,width=\textwidth]
    {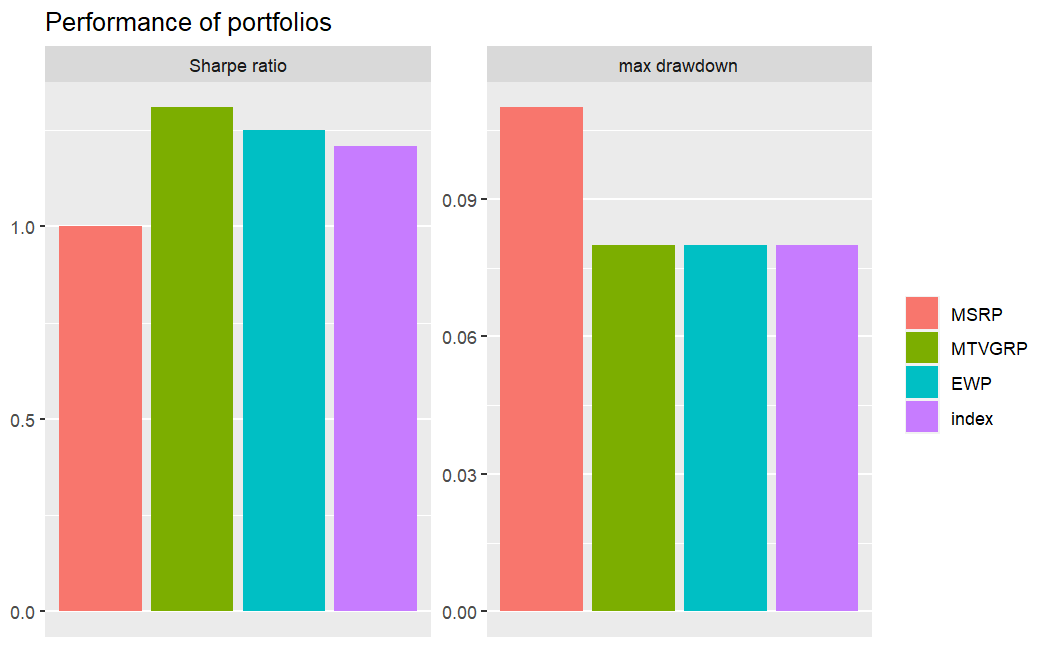}
\end{subfigure} \\
\begin{subfigure}[t]{0.22\textwidth}
   \includegraphics[trim=0 0 0 0,clip,width=\textwidth]
    {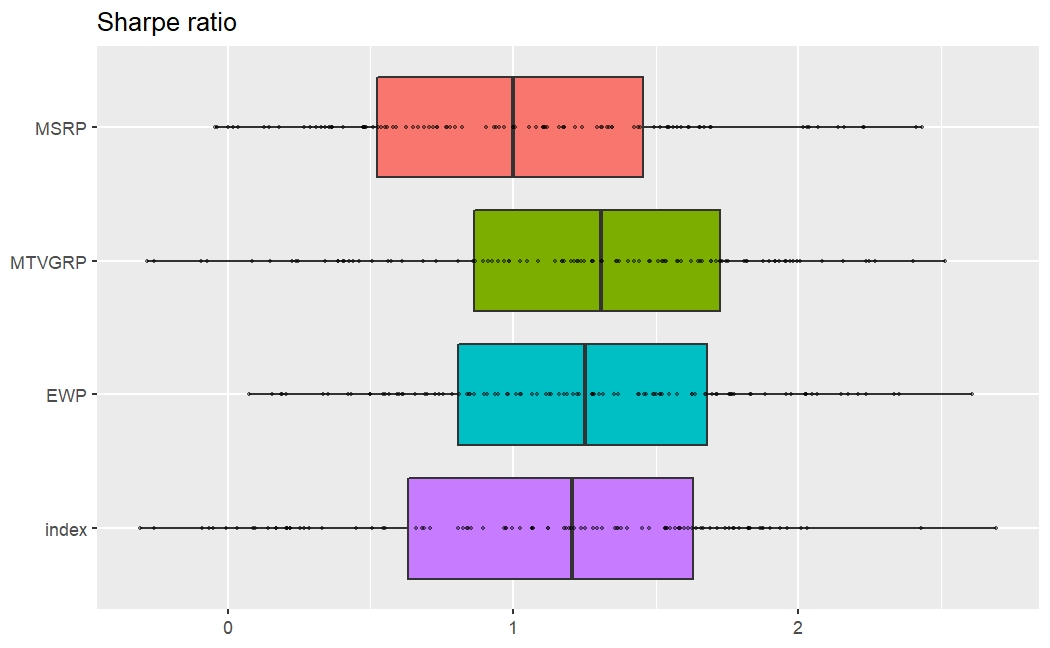}
\end{subfigure}
\begin{subfigure}[t]{0.22\textwidth}
   \includegraphics[trim=0 0 0 0,clip,width=\textwidth]
    {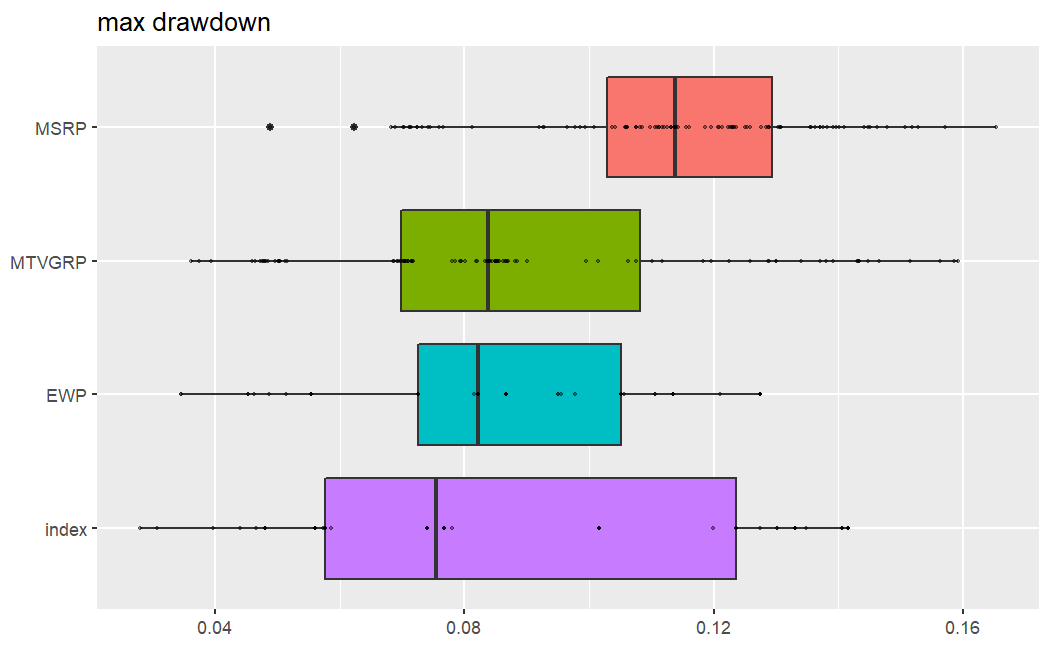}
\end{subfigure}
\caption{Sharpe ratio (left) and  the maximum draw-down (right) performance of the proposed MTVGRP portfolio compared to the market index, the  MSRP, and the EWP portfolios. The barplots on the top  represent the mean values of Sharpe ratio and maximum drawdown and the boxplots on the bottom  depict the distribution  of these performance measures. 
}
\label{fig:sr_barchart}
\end{figure}

\begin{table}[H]
\scriptsize
\centering
 \caption{Backtesting results of  the proposed MTVGRP compared to the MSRP and the EWP portfolios.}
 \label{table:backtest}
 \begin{tabular}{m{1cm}||m{1.1cm}|m{1cm}|m{0.9cm}|m{1.1cm}} 
     & Ann. Return  &  Ann. Volatility & Sharpe Ratio & Max Drawdown   \\
     \hline

     \multicolumn{1}{c||}{ MTVGRP (Proposed) }  & 
     \bf 0.14 & 
     \bf 0.11 & 
    \bf 1.31 &
    \bf  0.08 \\
     \hline
    
    \multicolumn{1}{c||}{ MSRP}   
     &  0.13 
     & 0.14
     &  1.00
     & 0.11 \\
     \hline

     \multicolumn{1}{c||}{ EWP (Uniform)}   
     &  0.13 
     &  0.11
     &   1.25
     & 0.08 \\
     \hline
 \end{tabular}
 
 \end{table}

\section{Conclusion}

This paper explores the problem of learning time-varying graphs specifically designed for heavy-tailed data. We propose a novel approach for time-varying graph learning that is tailored to infer graph-structured models capable of effectively capturing heavy-tailed distributions. Our proposed approach employs a probabilistic model to formulate the problem of learning time-varying graphs. We also incorporate spectral constraints into the problem, enabling us to learn multi-component graphs suitable for clustering. We present a solution based on a maximum-a-posteriori estimation framework using a semi-online strategy, wherein a single data frame is utilized to  update the graph. To demonstrate the effectiveness and robustness of our method in graphical modeling of time-varying heavy-tailed data, particularly within financial markets, we conduct  numerical experiments using both synthetic and real datasets.


\appendices
\section{Proof of Proposition \ref{prop:1}}
\label{App:1}
\begin{proof}

The MAP estimation rule can be  expressed as follows:
 \begin{equation*}
\begin{array}{cl}        
     \underset{  
 \wv_{n}\geq \zerov,\, \av\geq \zerov,\, \Xb_{n}}{\mathsf{min}} &
      -\log p(\wv_{n} ,\av, \Xb_{n}\vert {\wv}_{n-1}, \Yb_{n}, \Mb_n)= \\
      & -\log p( \Yb_{n}\vert\Xb_{n},\Mb_n) - \log p(\Xb_{n}\vert  \wv_{n}) \\
      &- \log p(\wv_{n}\vert {\wv}_{n-1},\av) -\log p(\av) + \cte \\
    \mathsf{s.t.} & \wv_{n} \in \Omega_\wv.
\end{array}
\end{equation*}

Now, given the Gaussian distribution of the measurement noise, we may write:
 \begin{align*}
 \begin{split}   
p(\Yb_n \vert \Xb_n,\Mb_n ) &= \hspace{-3pt}\prod_{t\in F_n} p(\yv_t\vert\xv_t,\mv_t) 
\\
&= C_0 \exp\left(-\frac{1}{2\sigma_n^2} \normop{\Yb_n -\Mb_n\odot  \Xb_n}_F^2\right), 
 \end{split}
\end{align*}
where $C_0$ is a constant.
Moreover, assuming the Student-\textit{t} distribution in \eqref{eq:heavy-tail}  for the data given the underlying graph model, we have:
 \begin{align*}
p(\Xb \vert \wv_{n}  ) &= \prod_{t\in F_n} p(\xv_t\vert\wv_{n}) \nonumber\\
&= C_1 \prod_{t\in F_n} \mathrm{det^*}(\Lc \wv_{n})\left(1+\frac{\xv_t^\top \Lc \wv_{n} \xv_t}{\nu}\right)^{-(\nu + p)/2}.
\end{align*}
with $C_1$ being another constant.

Now, let $\vv_n =  \av \odot\wv_{n-1} \geq \zerov$. 
Then, we may write:
\begin{align*}
     p(\wv_{n}\vert \wv_{n-1},\av) &= 
     p(\wv_{n}\vert \vv_n)=  \prod_i p(w_{n}(i)\vert v_n(i)).
\end{align*}
Using    the non-negative VAR equation $w_n(i) = ( v_n(i) + \epsilon_n(i) )_+$, we have:
\begin{align*}
     p(w_{n}(i)\vert v_n(i)) = \left\{
     \begin{array}{lc}
         \Prob[\epsilon_n(i)= w_n(i) -v_n(i)] &  w_n(i) >0 \\
         \Prob[\epsilon_n(i)\leq -v_n(i)]  &  w_n(i) =0
     \end{array}
     \right.
\end{align*}
Assuming i.i.d. Laplace distribution for the elements of $\epv_n$ as in \eqref{eq:laplace}, we get:
\begin{align}
     p(w_{n}(i)\vert v_n(i)) &= \left\{
     \begin{array}{lc}
         \frac{1}{2\sigma_\epsilon} \exp\left(-\frac{\vert w_n(i) -v_n(i) \vert }{\sigma_\epsilon}\right) &  w_n(i) >0 \\
         \frac{1}{2} \exp\left(\frac{- v_n(i)  }{\sigma_\epsilon}\right)  &  w_n(i) =0
     \end{array}
     \right. \nonumber\\
     &= \frac{1}{2 \sigma_\epsilon^{\Id(w_n(i) >0)}} \exp\left(-\frac{\vert w_n(i) -v_n(i) \vert }{\sigma_\epsilon}\right).\nonumber
\end{align}
Hence, we obtain:
\begin{align*}
     p(\wv_{n}\vert \wv_{n-1},\av) &= \frac{1}{2^p} \frac{1}{\sigma_\epsilon^{\normop{\wv_{n}}_0}} \exp\left(-\frac{1}{\sigma_\epsilon} \normop{\wv_{n}-\av\odot\wv_{n-1}}_1\right).
\end{align*}

Finally, plugging these probabilities in \eqref{eq:MAP_Online}, also considering 
%
exponential distribution for the VAR parameters $\av$ as in \eqref{eq:laplace},
%
we
can obtain \eqref{eq:problem_semi} after simplification.
\end{proof}

\section{Proof of Proposition \ref{prop:2}}
\label{App:2}
\begin{proof}
Let $g(\xv_t) = \log\left(1 + \dfrac{\xv^\top_{t}\Lc\wv^{l+1}_{n}\,{\xv_{t}}}{\nu} \right)$. Then, using the inequality $\log(x)\leq x-1,\,\, \forall x>0$, we have:
\begin{align}
\begin{array}{rl}
   g(\xv_t) &= \log\left(1 + \dfrac{\xv^\top_{t}\Lc\wv^{l+1}_{n}\,{\xv_{t}}}{\nu}  \right)\nonumber\\
  &\leq \log\left(1 + \dfrac{\xv^{l\top}_{t}\Lc\wv^{l+1}_{n}\,{\xv^l_{t}}}{\nu} \right) + \frac{\xv^\top_{t}\Lc\wv^{l+1}_{n}\,{\xv_{t}}+\nu}{\xv^{l\top}_{t}\Lc\wv^{l+1}_{n}\,{\xv^l_{t}}+\nu} - 1 \nonumber\\
  &=g(\xv_t^l) - 1 + \xv_t^\top \frac{\Lc \wv^{l+1}_n}{\xv^{l\top}_{t}\Lc\wv^{l+1}_{n}\,{\xv^l_{t}}+\nu} \xv_t + \frac{\nu}{\xv^{l\top}_{t}\Lc\wv^{l+1}_{n}\,{\xv^l_{t}}+\nu}\nonumber\\
  &= \frac{1}{\xv^{l\top}_{t}\Lc\wv^{l+1}_{n}\xv^l_t+\nu }  \xv_t^\top\Lc \wv^{l+1}_n \xv_t + h(\xv_t^l),
  \end{array}
\end{align} 
where $h(\xv_t^l)$ is a constant.
Then, we may write:
\begin{align*}
\begin{split}
    f_{\xv_t}(\xv_t) &= \frac{1}{T_n\sigma_n^2} \normop{\Diag(\mv_{t})\xv_{t}-\yv_{t}}^2 +\frac{p+\nu}{T_n} g(\xv_t)
    \\
    &\leq \frac{1}{T_n}\left( \xv_t^\top\Qb_t \xv_t - 2 \cv_t^\top \xv_t\right) + r(\xv_t^l),
    \end{split}
\end{align*}
 with $\Qb_t$ and $\cv_t$  given in \eqref{eq:Q_c} and $r(\xv_t^l) = \frac{p+\nu}{T_n} h(\xv_t^l) + \frac{\normop{\yv_t}^2}{T_n\sigma_n^2}$.
Now, for $\tau > \lambda_{\max}(\Qb_t)$, we may propose another majorization function as follows:
\begin{align*}
    f_\xv(\xv_t)  &\leq  \frac{1}{T_n}\Big(\xv_t^\top \Qb_t \xv_t+ (\xv_t - \xv^l_t)^\top (\tau \Ib - \Qb_t) (\xv_t-\xv_t^l) 
    \nonumber\\
    &\hspace{40pt}
    - 2 \cv_t^\top \xv_t\Big) + r(\xv_t^l) \nonumber\\
    &= \frac{\tau}{T_n} \normop{\xv_t - \xv_t^l + \frac{\Qb_t\xv_t^l - \cv_t}{\tau}}^2 + C(\xv^l_t),
\end{align*}
where $C(\xv_t^l)$ is a constant.
To find $\tau> \lambda_{\max}(\Qb_t)$, we write:
\begin{align}
\begin{array}{ll}
  \lambda_{\max}(\Qb_t)\!\!\!\! &= \lambda_{\max} \left( \frac{1}{\sigma_n^2} \Diag(\mv_t) + \frac{p+\nu}{\xv^{^l\top}_{t}\Lc\wv^{l+1}_{n}\,{\xv^l_{t}}+\nu} \Lc\wv^{l+1}_{n} \right) \nonumber\\
  &\leq \frac{1}{\sigma_n^2}\lambda_{\max}\left(\Diag(\mv_t) \right) + \frac{(p+\nu)\lambda_{\max}\left(\Lc\wv^{l+1}_{n}\right)}{\xv^{l\top}_{t}\Lc\wv^{l+1}_{n}{\xv_{t}^l}+\nu} \nonumber\\
  &=\frac{1}{\sigma_n^2}+ \frac{p+\nu}{\xv^{l\top}_{t}\Lc\wv^{l+1}_{n}{\xv_{t}^l}+\nu}\lambda_{\max}\left(\Lc\wv^{l+1}_{n}\right),
  \end{array}
\end{align}
 where we applied 
 Weyl's inequality \cite{laub_matrix_2005} in the first expression. Hence, it suffices to choose $\tau \geq \frac{1}{\sigma_n^2}+ \frac{p+\nu}{\xv^{l\top}_{t}\Lc\wv^{l+1}_{n}{\xv_{t}^l}+\nu}\lambda_{\max}\left(\Lc\wv^{l+1}_{n}\right)$.
\end{proof}

\section{Proof of Theorem \ref{thm:convergence}}
\label{App:3}
\begin{proof}
To prove Theorem
~\ref{thm:convergence},
 we first use the following Lemma   to establish the boundedness of the sequence $\big\{ \big(\Lb_{n}^l, \wv_{n}^l, \uv_{n}^l, \Xb_{n}^l, \av^l, \Vb_n^l, \Phib_{n}^l, \muv_{n}^l, \zv_{n}^l \big) \big\}$.

\begin{lemma}\label{lemma:3}
The sequence of the augmented Lagrangian  $\big\{ \big(\Lb_{n}^l, \wv_{n}^l, \uv_{n}^l, \Xb_{n}^l, \av^l, \Vb_n^l, \Phib_{n}^l, \muv_{n}^l, \zv_{n}^l \big) \big\}$ generated by 
Algorithm~1
is bounded.
\end{lemma}

\begin{proof}
Let $\Lb_{n}^0$, $\wv_{n}^0$, $\uv_{n}^0$, $\Xb_{n}^0$, $\av^0$, $\Vb_n^0$, $\Phib_{n}^0$, $\muv_{n}^0$, and $\zv_{n}^0$ be the initialization of the sequences $\big\{ \Lb_{n}^l \big\}$, $\big\{\wv_{n}^l \big\}$, $\big\{ \uv_{n}^l \big\}$, $\big\{\Xb_{n}^l\big\}$, $\big\{\av^l\big\}$, $\big\{\Vb_n^l\big\}$, $\big\{\Phib_{n}^l\big\}$, $\big\{\muv_{n}^l\big\}$, and $\big\{\zv_{n}^l \big\}$, respectively, and $\big\|\Lb_{n}^0\big\|_F$, $\big\|\wv_{n}^0\big\|$, $\big\|\uv_{n}^0\big\|$, $\big\|\Xb_{n}^0\big\|_F$, $\big\|\av^0\big\|$, $\big\|\Vb_n^0\big\|_F$, $\big\|\Phib_{n}^0\big\|_F$, $\big\|\muv_{n}^0\big\|$, and $\big\|\zv_{n}^0\big\|$ are bounded.

Recall that the sequence $\big\{\Lb_{n}^l \big\}$ is established by
\begin{equation}
\label{eq:Theta_update_proof}
 \Lb_{n}^{l} = \frac{1}{2} \Ub^{l-1} \Big(\Gamb^{l-1} + \big({ {\Gamb^{l-1}}^ 2 + \frac{ 4}{\rho} \Ib}\big)^{1/2}\Big){(\Ub^{l-1})}^\top,
\end{equation}
where $\Gamb^{l-1}$ contains the largest $p-k$ eigenvalues of $ \Lc \wv_{n}^{l-1} + \Phib_{n}^{l-1}/\rho $, and $\Ub^{l-1}$ contains the corresponding eigenvectors. When $l = 1$, $\big\|\Gamb^{0}\big\|_F$ is bounded since both $\big\| \wv_{n}^{0} \big\|_F$ and $\big\| \Phib_{n}^{0} \big\|_F$ are bounded. Therefore, we can conclude that $\big\|\Lb_{n}^{1}\big\|_F$ is bounded.

The sequence $\big\{\wv_{n}^{l} \big\}$ is established by
\begin{equation}
    \wv_{n}^{l} = \argminn_{\wv_{n}\geq \zerov} \,\,\,\frac{\rho(4p  - 1)}{2}\normop{\wv_{n} - \cv^{l-1}}^2 + \beta \normop{\wv_{n}}_0. 
    \label{eq:w_update_proof}
\end{equation}
Note that $\big\|\cv^{0} \big\|$ is bounded because $\big\|\Lb_{n}^{1}\big\|_F$, $\big\|\wv_{n}^0\big\|$, $\big\|\av^0\big\|$, $\big\|\Vb_n^0\big\|_F$, $\big\|\Phib_{n}^0\big\|_F$, $\big\|\muv_{n}^0\big\|$, and $\big\|\zv_{n}^0\big\|$ are bounded. We can check that the objective function $\frac{\rho(4p  - 1)}{2}\normop{\wv_{n} - \cv^{0}}^2 + \beta \normop{\wv_{n}}_0$ is coercive, which leads to the boundedness of $\big\|\wv_{n}^1\big\|$.

The sequence $\big\{\uv_{n}^l \big\}$ is established by
\begin{equation}
\label{eq:u_update_proof}
\uv_{n}^{l} =\argminn_{\uv } 
    \frac{\rho}{2}\normop{\uv_{n}- {\bf{b}}^{l}}^2 + \alpha\normop{\uv_{n}}_1,
\end{equation}
where ${\bf{b}}^{l} =  \wv^{l}_{n} - \av^{l-1}\odot\hat{\wv}_{n-1} - \frac{1}{\rho}  \muv_{n}^{l-1}$. We can check that $\big\| {\bf{b}}^{1} \big\|$ is bounded and the objective function in \eqref{eq:u_update_proof} is coercive, and thus $\big\| \uv_{n}^{1} \big\|$ is bounded. Similarly, by checking the coercivity, we can establish the boundedness of $\big\| \Xb_{n}^1 \big\|$ and $\big\| \av_{n}^1 \big\|$. 

Recall that the sequence $\big\{\Vb_n^{l} \big\}$ is updated by
\begin{equation}
\label{eq:V_update_proof}
    \Vb_n^{l}  = \hspace{-5pt}\underset{ \Vb_n\in \Real^{p\times k},\Vb_n^\top\Vb_n =\Ib  }{\mathsf{argmin}} ~ \hspace{-5pt}\tr\left(\Lc\wv_n^{l}\Vb_n\Vb_n^\top\right).
\end{equation}
The constraint $\Vb_n^\top\Vb_n =\Ib$ in \eqref{eq:V_update_proof} guarantees the boundedness of $\big\| \Vb_n^{l} \big\|_F$.

According to the dual variable updates, we obtain:
\begin{align}
\label{eq:dual_update_proof}
\begin{split}
\Phib_{n}^{1} &= \Phib_{n}^{0} + \rho \left(\Lc\wv_{n}^{1} - \Lb_{n}^{1}\right),  \\
 \muv_{n}^{1} &=  \muv_{n}^{0} + \rho \left( \uv_{n}^{1} - \wv_{n}^{1} + \av^{1}\odot\hat{\wv}_{n-1} \right),  \\
\zv_{n}^{1} &= \zv_{n}^{0} + \rho \left(\mathfrak{d}\wv_{n}^{1} - \dv \right). 
\end{split}
\end{align}
It is straightforward to conclude that $\big\| \Phib_{n}^{1} \big\|_F$, $\big\| \muv_{n}^{1} \big\|$, and $\big\| \zv_{n}^{1} \big\|$ are bounded. Therefore, it holds for $l =1$ that the sequence $\big\{ \big(\Lb_{n}^l, \wv_{n}^l, \uv_{n}^l, \Xb_{n}^l, \av^l, \Vb_n^l, \Phib_{n}^l, \muv_{n}^l, \zv_{n}^l \big) \big\}$ is bounded.

Now, we assume that $\big\{\!\big(\Lb_{n}^{l-1}, \wv_{n}^{l-1}, \uv_{n}^{l-1}, \Xb_{n}^{l-1}, \av^{l-1}, \Vb_n^{l-1},$ $\Phib_{n}^{l-1}, \muv_{n}^{l-1}, \zv_{n}^{l-1} \!\big) \big\}$ is bounded for some $l \geq 1$, and check the boundedness of $\big\{ \big(\Lb_{n}^{l}, \wv_{n}^{l}, \uv_{n}^{l}, \Xb_{n}^{l}, \av^{l}, \Vb_n^{l}, \Phib_{n}^{l}, \muv_{n}^{l}, \zv_{n}^{l} \big) \big\}$. Similar to the proof in \eqref{eq:Theta_update_proof}, we can prove that $\big\|\Lb_{n}^{l} \big\|_F$ is bounded. By checking the coercivity, we can also obtain the boundedness of $\big\|\wv_{n}^l\big\|$, $\big\| \uv_{n}^{l} \big\|$, $\big\| \Xb_{n}^l \big\|$, and $\big\| \av_{n}^l \big\|$. We can also obtain that $\big\| \Phib_{n}^{l} \big\|_F$, $\big\| \muv_{n}^{l} \big\|$, and $\big\| \zv_{n}^{l} \big\|$ are bounded according to the boundedness of $\big\| \Phib_{n}^{l-1} \big\|_F$, $\big\| \muv_{n}^{l-1} \big\|$, $\big\|\zv_{n}^{l-1} \big\|$, $\big\|\Lb_{n}^{l} \big\|_F$, $\big\|\wv_{n}^l\big\|$, $\big\| \uv_{n}^{l} \big\|$, and $\big\| \av_{n}^l \big\|$. Therefore, we establish the boundedness of the sequence $\big\{ \big(\Lb_{n}^l, \wv_{n}^l, \uv_{n}^l, \Xb_{n}^l, \av^l, \Vb_n^l, \Phib_{n}^l, \muv_{n}^l, \zv_{n}^l \big) \big\}$.

\end{proof}

Now, recall that
\begin{align}
    &L_\rho(\Lb_{n}^l, \wv_{n}^l, \uv_{n}^l, \Xb_{n}^l, \av^l, \Vb_n^l, \Phib_{n}^l, \muv_{n}^l, \zv_{n}^l) =
    \nonumber\\ 
    & \,\quad f(\wv_{n}^l,\Xb_{n}^l, \av^l, \uv_{n}^l, \Lb_{n}^l,\Vb_n^l) 
    + \big< \Phib_{n}^l, \Lc\wv_{n}^l-\Lb_{n}^l\big> \nonumber\\
    & \quad + \frac{\rho}{2}\normop{\Lc\wv_{n}^l-\Lb_{n}^l}_F^2 
    + \big<\muv_{n}^l,\uv_{n}^l -\wv_{n}^l + \av^l\odot \hat{\wv}_{n-1} \big> \nonumber\\
    & \quad + \frac{\rho}{2}\normop{\uv_{n}^l -\wv_{n}^l + \av^l\odot \hat{\wv}_{n-1}}^2 
    + \big<\zv_{n}^l, \mathfrak{d}\wv_{n}^l- \dv \big> \nonumber\\
    & \quad + \frac{\rho}{2}\normop{\mathfrak{d}\wv_{n}^l- \dv }^2,
\end{align}

where $f$ is defined as 
\begin{align}
    \begin{split}
    &f(\wv_{n},\Xb_{n}, \av, \uv_{n}, \Lb_{n}, \Vb_n
    )\triangleq\\
    & \hspace{5pt} - \log \det{\!^\ast} (\Lb_{n} ) + \alpha \normop{\uv_{n}}_1 + \beta \normop{\wv_{n}}_0  \\
     &\hspace{5pt}+\, \frac{\nu+p}{T_n} \sum_{t\in F_{n}}\log\left(1+\frac{\xv_t^\top\Lc\wv_{n}\xv_t}{\nu} \right) \\
      &\hspace{5pt}  + \frac{1}{T_n\sigma_n^2} \normop{\Yb_{n} - \Mb_{n}\odot \Xb_{n}}_F^2 + \gamma \av^\top \onev\\
      &\hspace{5pt} \,+\eta \tr\left(\Lc\wv_n\Vb_n\Vb_n^\top\right).
      \end{split}
\end{align}
The lower boundedness of the sequence $\big\{ L_\rho \big(\Lb_{n}^l, \wv_{n}^l, \uv_{n}^l, \Xb_{n}^l, \av^l, \Vb_n^l, \Phib_{n}^l, \muv_{n}^l, \zv_{n}^l \big) \big\}$ can be established by the boundedness of $\big\{ \big(\Lb_{n}^l, \wv_{n}^l, \uv_{n}^l, \Xb_{n}^l, \av^l, \Vb_n^l, \Phib_{n}^l, \muv_{n}^l, \zv_{n}^l \big) \big\}$ as shown in Lemma~\ref{lemma:3}.

The update of $\Lb_{n}^{l+1}$ is obtained by minimizing $L_\rho$ with respect to $\Lb$ with the other variables fixed as $\wv_{n}^l$, $\uv_{n}^l$, $\Xb_{n}^l$, $\av^l$, $\Vb_n^l$, $\Phib_{n}^l$, $\muv_{n}^l$, $\zv_{n}^l$, respectively. Therefore, we obtain:
\begin{align}  
\begin{split}
L_\rho(\Lb_{n}^{l+1}, \wv_{n}^l, \uv_{n}^l, \Xb_{n}^l, \av^l, \Vb_n^l, \Phib_{n}^l, \muv_{n}^l, \zv_{n}^l)   \\
\leq L_\rho(\Lb_{n}^l, \wv_{n}^l, \uv_{n}^l, \Xb_{n}^l, \av^l, \Vb_n^l, \Phib_{n}^l, \muv_{n}^l, \zv_{n}^l).
\end{split}
\end{align}
The update of $\wv_{n}^{l+1}$ employs the majorization-minimization technique, which ensures a reduction in the objective function value. Consequently, we have:
\begin{equation}
\begin{split}
    L_\rho(\Lb_{n}^{l+1}, \wv_{n}^{l+1}, \uv_{n}^l, \Xb_{n}^l, \av^l, \Vb_n^l, \Phib_{n}^l, \muv_{n}^l, \zv_{n}^l)  \\
    \leq L_\rho(\Lb_{n}^{l+1}, \wv_{n}^l, \uv_{n}^l, \Xb_{n}^l, \av^l, \Vb_n^l, \Phib_{n}^l, \muv_{n}^l, \zv_{n}^l).
    \end{split}
\end{equation}
Similar to the case of $\Lb_{n}^{l+1}$, the update of $\uv_{n}^{l+1}$ is achieved by minimizing $L_\rho$ with respect to $\uv$ with the other variables fixed as $\Lb_{n}^{l+1}$, $\wv_{n}^{l+1}$, $\Xb_{n}^l$, $\av^l$, $\Vb_n^l$, $\Phib_{n}^l$, $\muv_{n}^l$, $\zv_{n}^l$,respectively. Then we obtain:
\begin{equation}
\begin{split}
L_\rho(\Lb_{n}^{l+1}, \wv_{n}^{l+1}, \uv_{n}^{l+1}, \Xb_{n}^l, \av^l, \Vb_n^l, \Phib_{n}^l, \muv_{n}^l, \zv_{n}^l) \\
\leq L_\rho(\Lb_{n}^{l+1}, \wv_{n}^{l+1}, \uv_{n}^l, \Xb_{n}^l, \av^l, \Vb_n^l, \Phib_{n}^l, \muv_{n}^l, \zv_{n}^l).
\end{split}
\end{equation}
The update of $\Xb_{n}^{l+1}$ uses the majorization minimization technique, which guarantees the decrease of the objective function value, i.e.:
\begin{equation}
\begin{split}
L_\rho(\Lb_{n}^{l+1}, \wv_{n}^{l+1}, \uv_{n}^{l+1}, \Xb_{n}^{l+1}, \av^l, \Vb_n^l, \Phib_{n}^l, \muv_{n}^l, \zv_{n}^l)\\
\leq L_\rho(\Lb_{n}^{l+1}, \wv_{n}^{l+1}, \uv_{n}^{l+1}, \Xb_{n}^l, \av^l, \Vb_n^l, \Phib_{n}^l, \muv_{n}^l, \zv_{n}^l).
\end{split}
\end{equation}
Regarding the update of $\av^{l+1}$, one has
\begin{align}
\label{eq:proof_a} 
\begin{split}
& L_\rho(\Lb_{n}^{l+1}, \wv_{n}^{l+1}, \uv_{n}^{l+1}, \Xb_{n}^{l+1}, \av^{l+1}, \Vb_n^l, \Phib_{n}^l, \muv_{n}^l, \zv_{n}^l) \\
&\hspace{0pt}- L_\rho(\Lb_{n}^{l+1}, \wv_{n}^{l+1}, \uv_{n}^{l+1}, \Xb_{n}^{l+1}, \av^l, \Vb_n^l, \Phib_{n}^l, \muv_{n}^l, \zv_{n}^l) \\
 &= \rho\, \big< \big( \av^{l+1} - \av^{l}\big)\odot \hat{\wv}_{n-1}, \uv_{n}^{l+1} - \wv_{n}^{l+1} + \frac{1}{\rho} \muv_{n}^l \big> +
\\
&\quad\,  \frac{\rho}{2} \big\| \av^{l+1} \odot \hat{\wv}_{n-1}\big\|^2 
- \frac{\rho}{2} \big\| \av^{l} \odot \hat{\wv}_{n-1}\big\|^2 + \big< \av^{l+1} - \av^l, \gamma \onev\big>. 
\end{split}
\end{align}
Recall that $\av^{l+1}$ is the minimizer of the following optimization problem:
\begin{equation}\label{eq:update_a_proof}
\min_{\av\geq \zerov} 
    \,\frac{\rho}{2} \big\|\av\odot\hat{\wv}_{n-1}- \wv^{l+1}_{n} + \uv^{l+1}_{n} + \frac{1}{\rho}  \muv_{n}^l \big\|^2 + \gamma \av^\top \onev. 
\end{equation}
Therefore, $\av^{l+1}$ satisfies the following inequality:
\begin{equation}
\begin{split} 
\big< \nabla g(\av), {\bf{z}} - \av\big>
 \geq 0, \quad \forall {\bf{z}} \geq {\bf{0}},
\end{split}
\end{equation}
where $g$ is the objective function of Problem~\eqref{eq:update_a_proof}. By setting ${\bf z = \av^l}$ and $\av = \av^{l+1}$, we further obtain:
\begin{align}
\label{eq:proof-a-minimizer}
\begin{split}
&\rho \big< \av^{l+1} \odot \hat{\wv}_{n-1}, \big(\av^{l} - \av^{l+1} \big) \odot \hat{\wv}_{n-1}\big> \geq \big< \av^{l+1} - \av^l, \gamma \onev\big>  \\
&\hspace{50pt}+ \rho \big< \uv_{n}^{l+1} - \wv_{n}^{l+1} + \frac{1}{\rho} \muv_{n}^l, \big( \av^{l+1} - \av^{l}\big)\odot \hat{\wv}_{n-1}\big>.
\end{split}
\end{align}
Plugging \eqref{eq:proof-a-minimizer} into \eqref{eq:proof_a}, we get:
\begin{align}
\begin{split}
& L_\rho(\Lb_{n}^{l+1}, \wv_{n}^{l+1}, \uv_{n}^{l+1}, \Xb_{n}^{l+1}, \av^{l+1}, \Vb_n^l, \Phib_{n}^l, \muv_{n}^l, \zv_{n}^l) \\
&\quad- L_\rho(\Lb_{n}^{l+1}, \wv_{n}^{l+1}, \uv_{n}^{l+1}, \Xb_{n}^{l+1}, \av^l, \Vb_n^l, \Phib_{n}^l, \muv_{n}^l, \zv_{n}^l)  \\
& \leq \rho\, \big< \av^{l+1} \odot \hat{\wv}_{n-1}, \av^{l} \odot \hat{\wv}_{n-1}\big> - \frac{\rho}{2} \big\| \av^{l+1} \odot \hat{\wv}_{n-1}\big\|^2  \\
&\quad- \frac{\rho}{2} \big\| \av^{l} \odot \hat{\wv}_{n-1}\big\|^2  = - \frac{\rho}{2} \big\| \av^{l+1} \odot \hat{\wv}_{n-1} - \av^l \odot \hat{\wv}_{n-1} \big\|^2.
\end{split}
\end{align}
Since the update of $\Vb_n^{l+1}$ is achieved by minimizing $L_\rho$ with respect to $\Vb_n$ while keeping the other variables fixed, we can obtain:
\begin{equation}
\begin{split}
L_\rho(\Lb_{n}^{l+1}, \wv_{n}^{l+1}, \uv_{n}^{l+1}, \Xb_{n}^{l+1}, \av^{l+1}, \Vb_n^{l+1}, \Phib_{n}^l, \muv_{n}^l, \zv_{n}^l) \\
\leq L_\rho(\Lb_{n}^{l+1}, \wv_{n}^{l+1}, \uv_{n}^{l+1}, \Xb_{n}^{l+1}, \av^{l+1}, \Vb_n^l, \Phib_{n}^l, \muv_{n}^l, \zv_{n}^l).
\end{split}
\end{equation}
Now, consider the update formulas for the dual variables, $\Phib_{n}^{l+1}$, $\muv_{n}^{l+1}$, and $\zv_{n}^{l+1}$:
\begin{align}
\label{eq:dual_update}
\begin{split}
\Phib_{n}^{l+1} &= \Phib_{n}^{l} + \rho \left(\Lc\wv_{n}^{l+1} - \Lb_{n}^{l+1}\right),  \\
 \muv_{n}^{l+1} &=  \muv_{n}^{l} + \rho \left( \uv_{n}^{l+1} - \wv_{n}^{l+1} + \av^{l+1}\odot\hat{\wv}_{n-1} \right),  \\
\zv_{n}^{l+1} &= \zv_{n}^{l} + \rho \left(\mathfrak{d}\wv_{n}^{l+1} - \dv \right). 
\end{split}
\end{align}
According to
\eqref{eq:dual_update}, we obtain:
\begin{align}
\begin{split}
&L_\rho(\Lb_{n}^{l+1}, \wv_{n}^{l+1}, \uv_{n}^{l+1}, \Xb_{n}^{l+1}, \av^{l+1}, \Vb_n^{l+1}, \Phib_{n}^{l+1}, \muv_{n}^{l+1}, \zv_{n}^{l+1}) \\
&\quad= L_\rho(\Lb_{n}^{l+1}, \wv_{n}^{l+1}, \uv_{n}^{l+1}, \Xb_{n}^{l+1}, \av^{l+1}, \Vb_n^{l+1}, \Phib_{n}^l, \muv_{n}^l, \zv_{n}^l) \nonumber\\
& \quad\quad + \frac{1}{\rho} \big( \big\| \Phib_{n}^{l+1} - \Phib_{n}^l \big\|_F^2 + \big\| \muv_{n}^{l+1} - \muv_{n}^l \big\|^2 + \big\| \zv_{n}^{l+1} - \zv_{n}^l \big\|^2 \big).
\end{split}
\end{align}
Combining the above inequalities yields:
\begin{equation}\label{eq:sufficient-decrease}
\begin{split}
&L_\rho(\Lb_{n}^{l+1}, \wv_{n}^{l+1}, \uv_{n}^{l+1}, \Xb_{n}^{l+1}, \av^{l+1}, \Vb_n^{l+1}, \Phib_{n}^{l+1}, \muv_{n}^{l+1}, \zv_{n}^{l+1})\\
&\quad\leq L_\rho(\Lb_{n}^{l}, \wv_{n}^{l}, \uv_{n}^{l}, \Xb_{n}^{l}, \av^{l}, \Vb_n^{l}, \Phib_{n}^l, \muv_{n}^l, \zv_{n}^l) \\
& \quad \quad+ \frac{1}{\rho} \big( \big\| \Phib_{n}^{l+1} - \Phib_{n}^l \big\|_F^2 + \big\| \muv_{n}^{l+1} - \muv_{n}^l \big\|^2 + \big\| \zv_{n}^{l+1}\\
& \quad \quad- \zv_{n}^l \big\|^2 \big) - \frac{\rho}{2} \big\| \av^{l+1} \odot \hat{\wv}_{n-1} - \av^l \odot \hat{\wv}_{n-1} \big\|^2.
\end{split}
\end{equation}
By calculation, we obtain that if $\rho$ is sufficiently large such that
\begin{equation}
\label{eq:condition-rho}
\rho \geq \max_l \tfrac{c \big( \big\| \Phib_{n}^{l+1} - \Phib_{n}^l \big\|_F^2 + \big\| \muv_{n}^{l+1} - \muv_{n}^l \big\|^2 + \big\| \zv_{n}^{l+1} - \zv_{n}^l \big\|^2 \big)^{\frac{1}{2}}}{\big\| \av^{l+1} \odot \hat{\wv}_{n-1} - \av^l \odot \hat{\wv}_{n-1} \big\|}
\end{equation}
holds with some constant $c > \sqrt{2}$, then we have
\begin{equation}
\label{eq:seq-descreasing}
\begin{split}
L_\rho(\Lb_{n}^{l+1}, \wv_{n}^{l+1}, \uv_{n}^{l+1}, \Xb_{n}^{l+1}, \av^{l+1}, \Vb_n^{l+1}, \Phib_{n}^{l+1}, \muv_{n}^{l+1}, \zv_{n}^{l+1})\\
\leq L_\rho(\Lb_{n}^{l}, \wv_{n}^{l}, \uv_{n}^{l}, \Xb_{n}^{l}, \av^{l}, \Vb_n^{l}, \Phib_{n}^l, \muv_{n}^l, \zv_{n}^l),
\end{split}
\end{equation}
for any $l \in \mathbb{N}_+$. Therefore, the sequence $\big\{ L_\rho(\Lb_{n}^{l}, \wv_{n}^{l}, \uv_{n}^{l}, \Xb_{n}^{l}, \av^{l}, \Vb_n^{l}, \Phib_{n}^l, \muv_{n}^l, \zv_{n}^l)\big\}$ is monotonically decreasing and lower bounded, indicating that it is convergent.

By combining \eqref{eq:sufficient-decrease} and \eqref{eq:condition-rho}, we further obtain:
\begin{equation}\label{eq:limit}
\begin{split}
&L_\rho(\Lb_{n}^{l}, \wv_{n}^{l}, \uv_{n}^{l}, \Xb_{n}^{l}, \av^{l}, \Vb_n^{l}, \Phib_{n}^l, \muv_{n}^l, \zv_{n}^l) \\
&\quad- L_\rho(\Lb_{n}^{l+1}, \wv_{n}^{l+1}, \uv_{n}^{l+1}, \Xb_{n}^{l+1}, \av^{l+1}, \Vb_n^{l+1}, \Phib_{n}^{l+1}, \muv_{n}^{l+1}, \zv_{n}^{l+1})  \\
  &  \geq \frac{(c^2 - 2)\rho}{2} \big( \big\| \Lc\wv_{n}^{l+1}-\Lb_{n}^{l+1} \big\|_F^2 + \big\| \mathfrak{d}\wv_{n}^{l+1}- \dv \big\|^2 \\
  &\quad + \big\| \uv_{n}^{l+1} -\wv_{n}^{l+1} + \av^{l+1}\odot \hat{\wv}_{n-1} \big\|^2 \big).
\end{split}
\end{equation}
The convergence of $L_\rho(\Lb_{n}^{l}, \wv_{n}^{l}, \uv_{n}^{l}, \Xb_{n}^{l}, \av^{l}, \Vb_n^{l}, \Phib_{n}^l, \muv_{n}^l, \zv_{n}^l)$ yields:
\begin{align}
&\lim_{l \to + \infty} \big\| \Lc\wv_{n}^{l}-\Lb_{n}^{l} \big\|_F = 0,\nonumber\\
& \lim_{l \to + \infty} \big\| \mathfrak{d}\wv_{n}^{l}- \dv \big\| = 0,  \nonumber\\
&\lim_{l \to + \infty} \big\| \uv_{n}^{l} -\wv_{n}^{l} + \av^{l}\odot \hat{\wv}_{n-1} \big\| = 0, \nonumber
\end{align}
completing the proof.
\end{proof}

\bibliography{My_Library}

\begin{thebibliography}{10}

\bibitem{javaheri_EUSIPCO}
Amirhossein Javaheri and Daniel~P. Palomar,
\newblock ``Learning time-varying graphs for heavy-tailed data clustering,''
\newblock in {\em 2024 32nd European Signal Processing Conference (EUSIPCO)},
  2024, pp. 2472--2476.

\bibitem{campbell_social_2013}
W.~Campbell, C.~Dagli, and C.~Weinstein,
\newblock ``Social {Network} {Analysis} with {Content} and {Graphs},''
\newblock {\em Lincoln Laboratory Journal}, vol. 20, no. 1, pp. 62--81, 2013.

\bibitem{zhang_improved_2021}
X.~Zhang, Y.~Sun, H.~Liu, Z.~Hou, F.~Zhao, and C.~Zhang,
\newblock ``Improved clustering algorithms for image segmentation based on
  non-local information and back projection,''
\newblock {\em Information Sciences}, vol. 550, pp. 129--144, Mar. 2021.

\bibitem{tamura_applications_1970}
H.~Tamura, K.~Nakano, M.~Sengoku, and S.~Shinoda,
\newblock ``On {Applications} of {Graph}/{Network} {Theory} to {Problems} in
  {Communication} {Systems},''
\newblock {\em ECTI Transactions on Computer and Information Technology
  (ECTI-CIT)}, vol. 5, no. 1, pp. 15--21, Jan. 1970.

\bibitem{cardoso_algorithms_2020}
J.~V. de~M. Cardoso and D.~P. Palomar,
\newblock ``Learning undirected graphs in financial markets,''
\newblock in {\em 2020 54th Asilomar Conference on Signals, Systems, and
  Computers}, 2020, pp. 741--745.

\bibitem{ortega_graph_2018}
A.~Ortega, P.~Frossard, J.~Kovačević, J.~M.~F. Moura, and P.~Vandergheynst,
\newblock ``Graph {Signal} {Processing}: {Overview}, {Challenges}, and
  {Applications},''
\newblock {\em Proceedings of the IEEE}, vol. 106, no. 5, pp. 808--828, May
  2018.

\bibitem{dong_learning_2019}
X.~Dong, D.~Thanou, M.~Rabbat, and P.~Frossard,
\newblock ``Learning graphs from data: {A} signal representation perspective,''
\newblock {\em IEEE Signal Processing Magazine}, vol. 36, no. 3, pp. 44--63,
  May 2019.

\bibitem{rue_gaussian_2005}
H.~Rue and L.~Held,
\newblock {\em Gaussian {Markov} random fields: theory and applications},
\newblock Number 104 in Monographs on statistics and applied probability.
  Chapman \& Hall/CRC, Boca Raton, 2005.

\bibitem{ying2020nonconvex}
J.~Ying, J.~V. de~M. Cardoso, and D.~P. Palomar,
\newblock ``Nonconvex sparse graph learning under {Laplacian} constrained
  graphical model,''
\newblock in {\em Advances in Neural Information Processing Systems}, 2020,
  vol.~33, pp. 7101--7113.

\bibitem{ying2021minimax}
J.~Ying, J.~V. de~M. Cardoso, and D.~P. Palomar,
\newblock ``Minimax estimation of {Laplacian} constrained precision matrices,''
\newblock in {\em International Conference on Artificial Intelligence and
  Statistics}, 2021, vol. 130, pp. 3736--3744.

\bibitem{egilmez_graph_2017}
H.~E. Egilmez, E.~Pavez, and A.~Ortega,
\newblock ``Graph {Learning} {From} {Data} {Under} {Laplacian} and {Structural}
  {Constraints},''
\newblock {\em IEEE Journal of Selected Topics in Signal Processing}, vol. 11,
  no. 6, pp. 825--841, Sept. 2017.

\bibitem{zhang_graph_2016}
C.~Zhang, D.~Florencio, and P.~A. Chou,
\newblock ``Graph signal processing - a probabilistic framework,''
\newblock Tech. {R}ep. MSR-TR-2015-31, April 2015.

\bibitem{ying2020does}
J.~Ying, J.~V. de~M. Cardoso, and D.~P. Palomar,
\newblock ``Does the $\ell_1$-norm learn a sparse graph under {Laplacian}
  constrained graphical models?,''
\newblock {\em arXiv preprint arXiv:2006.14925}, 2020.

\bibitem{friedman_sparse_2008}
J.~Friedman, T.~Hastie, and R.~Tibshirani,
\newblock ``Sparse inverse covariance estimation with the graphical lasso,''
\newblock {\em Biostatistics}, vol. 9, no. 3, pp. 432--441, July 2008.

\bibitem{lake_discovering_2010}
B.~Lake and J.~Tenenbaum,
\newblock ``Discovering structure by learning sparse graphs,''
\newblock in {\em Proceedings of the 32nd {Annual} {Meeting} of the {Cognitive}
  {Science} {Society}}, Portland, Oregon, United States, Aug. 2010, pp.
  778--784.

\bibitem{zhao_optimization_2019}
L.~Zhao, Y.~Wang, S.~Kumar, and D.~P. Palomar,
\newblock ``Optimization {Algorithms} for {Graph} {Laplacian} {Estimation} via
  {ADMM} and {MM},''
\newblock {\em IEEE Transactions on Signal Processing}, vol. 67, no. 16, pp.
  4231--4244, Aug. 2019.

\bibitem{cardoso2022learning}
J.~V. de~M. Cardoso, J.~Ying, and D.~P. Palomar,
\newblock ``Learning bipartite graphs: {Heavy} tails and multiple components,''
\newblock in {\em Advances in Neural Information Processing Systems}, 2022,
  vol.~35, pp. 14044--14057.

\bibitem{kumar_unified_2020}
S.~Kumar, J.~Ying, J.~V. de~M. Cardoso, and D.~P. Palomar,
\newblock ``A {Unified} {Framework} for {Structured} {Graph} {Learning} via
  {Spectral} {Constraints},''
\newblock {\em Journal of Machine Learning Research}, vol. 21, no. 22, pp.
  1--60, 2020.

\bibitem{kaplan_structural_2009}
D.~Kaplan,
\newblock {\em Structural {Equation} {Modeling} (2nd ed.): {Foundations} and
  {Extensions}},
\newblock SAGE Publications, Inc., 2455 Teller Road, Thousand
  Oaks California 91320 United States, 2009.

\bibitem{songsiri_topology_2010}
J.~Songsiri and L.~Vandenberghe,
\newblock ``Topology {Selection} in {Graphical} {Models} of {Autoregressive}
  {Processes},''
\newblock {\em Journal of Machine Learning Research}, vol. 11, no. 91, pp.
  2671--2705, 2010.

\bibitem{bolstad_causal_2011}
A.~Bolstad, B.~D. Van~Veen, and R.~Nowak,
\newblock ``Causal {Network} {Inference} {Via} {Group} {Sparse}
  {Regularization},''
\newblock {\em IEEE Transactions on Signal Processing}, vol. 59, no. 6, pp.
  2628--2641, June 2011.

\bibitem{mei_signal_2017}
J.~Mei and J.~M.~F. Moura,
\newblock ``Signal {Processing} on {Graphs}: {Causal} {Modeling} of
  {Unstructured} {Data},''
\newblock {\em IEEE Transactions on Signal Processing}, vol. 65, no. 8, pp.
  2077--2092, Apr. 2017.

\bibitem{javaheri_learning_2024}
A.~Javaheri, A.~Amini, F.~Marvasti, and D.~P. Palomar,
\newblock ``Learning {Spatiotemporal} {Graphical} {Models} {From} {Incomplete}
  {Observations},''
\newblock {\em IEEE Transactions on Signal Processing}, vol. 72, pp.
  1361--1374, 2024.

\bibitem{kalofolias_learning_2017}
V.~Kalofolias, A.~Loukas, D.~Thanou, and P.~Frossard,
\newblock ``Learning time varying graphs,''
\newblock in {\em 2017 {IEEE} {International} {Conference} on {Acoustics},
  {Speech} and {Signal} {Processing} ({ICASSP})}, New Orleans, LA, Mar. 2017,
  pp. 2826--2830.

\bibitem{hallac_network_2017}
D.~Hallac, Y.~Park, S.~Boyd, and J.~Leskovec,
\newblock ``Network inference via the time-varying graphical lasso,''
\newblock in {\em Proceedings of the 23rd ACM SIGKDD International Conference
  on Knowledge Discovery and Data Mining}, New York, NY, USA, 2017, p.
  205–213.

\bibitem{yamada_time-varying_2020}
K.~Yamada, Y.~Tanaka, and A.~Ortega,
\newblock ``Time-{Varying} {Graph} {Learning} with {Constraints} on {Graph}
  {Temporal} {Variation},'' Jan. 2020,
\newblock arXiv:2001.03346.

\bibitem{cardoso_learning_2020}
J.~V. de~M. Cardoso and D.~P. Palomar,
\newblock ``Learning {Undirected} {Graphs} in {Financial} {Markets},''
\newblock in {\em 2020 54th {Asilomar} {Conference} on {Signals}, {Systems},
  and {Computers}}, Nov. 2020, pp. 741--745.

\bibitem{natali_online_2021-1}
A.~Natali, E.~Isufi, M.~Coutino, and G.~Leus,
\newblock ``Online {Graph} {Learning} {From} {Time}-{Varying} {Structural}
  {Equation} {Models},''
\newblock in {\em 2021 55th {Asilomar} {Conference} on {Signals}, {Systems},
  and {Computers}}, Oct. 2021, pp. 1579--1585,
\newblock ISSN: 2576-2303.

\bibitem{hamon_tracking_2013}
R.~Hamon, P.~Borgnat, P.~Flandrin, and C.~Robardet,
\newblock ``Tracking of a dynamic graph using a signal theory approach :
  application to the study of a bike sharing system,''
\newblock Sept. 2013, p. 101.

\bibitem{harary_dynamic_1997}
F.~Harary and G.~Gupta,
\newblock ``Dynamic graph models,''
\newblock {\em Mathematical and Computer Modelling}, vol. 25, no. 7, pp.
  79--87, Apr. 1997.

\bibitem{yamada_time-varying_2019}
K.~Yamada, Y.~Tanaka, and A.~Ortega,
\newblock ``Time-varying {Graph} {Learning} {Based} on {Sparseness} of
  {Temporal} {Variation},''
\newblock in {\em {ICASSP} 2019 - 2019 {IEEE} {International} {Conference} on
  {Acoustics}, {Speech} and {Signal} {Processing} ({ICASSP})}, May 2019, pp.
  5411--5415.

\bibitem{dong_learning_2016}
X.~Dong, D.~Thanou, P.~Frossard, and P.~Vandergheynst,
\newblock ``Learning {Laplacian} {Matrix} in {Smooth} {Graph} {Signal}
  {Representations},''
\newblock {\em IEEE Transactions on Signal Processing}, vol. 64, no. 23, pp.
  6160--6173, Dec. 2016.

\bibitem{natali_learning_2022}
A.~Natali, E.~Isufi, M.~Coutino, and G.~Leus,
\newblock ``Learning {Time}-{Varying} {Graphs} {From} {Online} {Data},''
\newblock {\em IEEE Open Journal of Signal Processing}, vol. 3, pp. 212--228,
  2022.

\bibitem{simonetto_class_2016}
A.~Simonetto, A.~Mokhtari, A.~Koppel, G.~Leus, and A.~Ribeiro,
\newblock ``A {Class} of {Prediction}-{Correction} {Methods} for
  {Time}-{Varying} {Convex} {Optimization},''
\newblock {\em IEEE Transactions on Signal Processing}, vol. 64, no. 17, pp.
  4576--4591, Sept. 2016.

\bibitem{baingana_tracking_2017}
B.~Baingana and G.~B. Giannakis,
\newblock ``Tracking {Switched} {Dynamic} {Network} {Topologies} {From}
  {Information} {Cascades},''
\newblock {\em IEEE Transactions on Signal Processing}, vol. 65, no. 4, pp.
  985--997, Feb. 2017.

\bibitem{vlaski_online_2018}
S.~Vlaski, H.~P. Maretić, R.~Nassif, P.~Frossard, and A.~H. Sayed,
\newblock ``{Online} {Graph} {Learning} {from} {Sequential} {Data},''
\newblock in {\em 2018 {IEEE} {Data} {Science} {Workshop} ({DSW})}, June 2018,
  pp. 190--194.

\bibitem{shafipour_online_2020}
R.~Shafipour and G.~Mateos,
\newblock ``Online {Topology} {Inference} from {Streaming} {Stationary} {Graph}
  {Signals} with {Partial} {Connectivity} {Information},''
\newblock {\em Algorithms}, vol. 13, no. 9, pp. 228, Sept. 2020.

\bibitem{money_online_2021}
R.~Money, J.~Krishnan, and B.~Beferull-Lozano,
\newblock ``Online {Non}-linear {Topology} {Identification} from
  {Graph}-connected {Time} {Series},''
\newblock in {\em 2021 {IEEE} {Data} {Science} and {Learning} {Workshop}
  ({DSLW})}, June 2021, pp. 1--6.

\bibitem{saboksayr_online_2021}
S.~S. Saboksayr, G.~Mateos, and M.~Cetin,
\newblock ``Online {Graph} {Learning} under {Smoothness} {Priors},''
\newblock {\em 2021 29th European Signal Processing Conference (EUSIPCO)}, pp.
  1820--1824, Aug. 2021.

\bibitem{sardellitti_online_2021}
S.~Sardellitti, S.~Barbarossa, and P.~Di~Lorenzo,
\newblock ``Online learning of time-varying signals and graphs,''
\newblock in {\em ICASSP 2021 - 2021 IEEE International Conference on
  Acoustics, Speech and Signal Processing (ICASSP)}, 2021, pp. 5230--5234.

\bibitem{buciulea_online_2024}
A.~Buciulea, M.~Navarro, S.~Rey, S.~Segarra, and A.~G. Marques,
\newblock ``Online {Network} {Inference} from {Graph}-{Stationary} {Signals}
  with {Hidden} {Nodes},'' Sept. 2024,
\newblock arXiv:2409.08760.

\bibitem{resnick_heavy-tail_2007}
S.~I. Resnick,
\newblock {\em Heavy-tail phenomena: probabilistic and statistical modeling},
\newblock Springer series in operations research and financial engineering.
  Springer, New York, NY [Heidelberg], 2007.

\bibitem{paratte_graph-based_2017}
J.~Paratte,
\newblock ``Graph-based {Methods} for {Visualization} and {Clustering},''
\newblock 2017,
\newblock Publisher: Lausanne, EPFL.

\bibitem{javaheri_joint_2024}
A.~Javaheri, A.~Amini, F.~Marvasti, and D.~P. Palomar,
\newblock ``Joint {Signal} {Recovery} and {Graph} {Learning} from {Incomplete}
  {Time}-{Series},''
\newblock in {\em {ICASSP} 2024 - 2024 {IEEE} {International} {Conference} on
  {Acoustics}, {Speech} and {Signal} {Processing} ({ICASSP})}, Apr. 2024, pp.
  13511--13515.

\bibitem{de_miranda_cardoso_graphical_2021}
J.~V. de~M.~Cardoso, J.~Ying, and D.~P. Palomar,
\newblock ``Graphical {Models} in {Heavy}-{Tailed} {Markets},''
\newblock in {\em Advances in {Neural} {Information} {Processing} {Systems}},
  2021, vol.~34, pp. 19989--20001.

\bibitem{boyd_distributed_2010}
S.~Boyd,
\newblock ``Distributed {Optimization} and {Statistical} {Learning} via the
  {Alternating} {Direction} {Method} of {Multipliers},''
\newblock {\em Foundations and Trends® in Machine Learning}, vol. 3, no. 1,
  pp. 1--122, 2010.

\bibitem{sun_majorization-minimization_2017}
Y.~Sun, P.~Babu, and D.~P. Palomar,
\newblock ``Majorization-{Minimization} {Algorithms} in {Signal} {Processing},
  {Communications}, and {Machine} {Learning},''
\newblock {\em IEEE Transactions on Signal Processing}, vol. 65, no. 3, pp.
  794--816, Feb. 2017.

\bibitem{tibshirani_regression_1996}
R.~Tibshirani,
\newblock ``Regression {Shrinkage} and {Selection} {Via} the {Lasso},''
\newblock {\em Journal of the Royal Statistical Society: Series B
  (Methodological)}, vol. 58, no. 1, pp. 267--288, 1996.

\bibitem{nie_constrained_2016}
F.~Nie, X.~Wang, M.~I. Jordan, and H.~Huang,
\newblock ``The {Constrained} {Laplacian} {Rank} algorithm for graph-based
  clustering,''
\newblock in {\em Proceedings of the {Thirtieth} {AAAI} {Conference} on
  {Artificial} {Intelligence}}, Phoenix, Arizona, Feb. 2016, {AAAI}'16, pp.
  1969--1976.

\bibitem{javaheri_graph_2023}
A.~Javaheri, J.~V. de~M.~Cardoso, and D.~P. Palomar,
\newblock ``Graph {Learning} for {Balanced} {Clustering} of {Heavy}-{Tailed}
  {Data},''
\newblock in {\em 2023 {IEEE} 9th {International} {Workshop} on {Computational}
  {Advances} in {Multi}-{Sensor} {Adaptive} {Processing} ({CAMSAP})},
  Herradura, Costa Rica, Dec. 2023, pp. 481--485.

\bibitem{everitt_cluster_2011}
B.~Everitt, Ed.,
\newblock {\em Cluster analysis},
\newblock Wiley series in probability and statistics. Wiley, Chichester, West
  Sussex, U.K, 5th ed edition, 2011.

\bibitem{newman_modularity_2006}
M.~E.~J. Newman,
\newblock ``Modularity and community structure in networks,''
\newblock {\em Proceedings of the National Academy of Sciences}, vol. 103, no.
  23, pp. 8577--8582, June 2006.

\bibitem{rand_objective_1971}
W.~M. Rand,
\newblock ``Objective {Criteria} for the {Evaluation} of {Clustering}
  {Methods},''
\newblock {\em Journal of the American Statistical Association}, vol. 66, no.
  336, pp. 846--850, Dec. 1971.

\bibitem{ng_spectral_2001}
A.~Ng, M.~Jordan, and Y.~Weiss,
\newblock ``On {Spectral} {Clustering}: {Analysis} and an algorithm,''
\newblock in {\em Advances in {Neural} {Information} {Processing} {Systems}},
  2001, NIPS'01, p. 849–856.

\bibitem{macqueen_methods_1967}
J.~MacQueen,
\newblock ``Some methods for classification and analysis of multivariate
  observations,''
\newblock in {\em Proc. 5th Berkeley Symp. Math. Stat. Probab.}, 1967, pp.
  481--485.

\bibitem{sharpe_mutual_1966}
W.~F. Sharpe,
\newblock ``Mutual {Fund} {Performance},''
\newblock {\em The Journal of Business}, vol. 39, no. 1, pp. 119--138, 1966.

\bibitem{schaible_parameter-free_1974}
S.~Schaible,
\newblock ``Parameter-free convex equivalent and dual programs of fractional
  programming problems,''
\newblock {\em Zeitschrift für Operations Research}, vol. 18, no. 5, pp.
  187--196, Oct. 1974.

\bibitem{laub_matrix_2005}
A.~J. Laub,
\newblock {\em Matrix analysis for scientists and engineers},
\newblock Society for Industrial and Applied Mathematics, Philadelphia, 2005.

\end{thebibliography}
\bibliographystyle{IEEEbib}



\end{document}